\documentclass[runningheads]{llncs}
\pdfoutput=1

\usepackage[T1]{fontenc}
\usepackage{graphicx}

\usepackage{ellipsis, ragged2e}
\usepackage[l2tabu, orthodox]{nag}

\usepackage[english]{babel}
\usepackage[T1]{fontenc}
\usepackage[utf8]{inputenc}
\usepackage[final]{microtype}

\usepackage{amsmath,amssymb}
\usepackage{csquotes}
\usepackage{booktabs}
\usepackage[most]{tcolorbox}
\usepackage{multirow}
\usepackage{adjustbox}
\usepackage{caption}
\usepackage{subcaption}
\usepackage{orcidlink}
\usepackage{hyperref}

\usepackage{pgfplots}
\pgfplotsset{compat=1.16}
\usepackage{tikz}
\usetikzlibrary{shapes,positioning,arrows,arrows.meta,calc,automata,matrix,fit,backgrounds}
\tikzstyle{state}+=[minimum size = 6mm, inner sep=0,outer sep=1]

\colorlet{disabled}{lightgray}
\tikzstyle{state}=[draw,rectangle,inner sep=5pt,rounded corners=2pt]
\tikzstyle{action}=[font=\small,inner sep=0pt,outer sep=3pt]
\tikzstyle{actionnode}=[circle,draw=black,fill=black,minimum size=1mm,inner sep=0,outer sep=0]
\tikzstyle{actionedge}=[draw,-]
\tikzstyle{prob}=[font=\scriptsize,inner sep=0pt,outer sep=1pt]
\tikzstyle{probedge}=[draw,->]
\tikzstyle{directedge}=[draw,->]
\tikzset{chainarrow/.tip={Stealth[length=3pt]}}
\tikzset{>=chainarrow}

\usepackage{xparse}
\usepackage{mathtools}
\usepackage{environ}
\usepackage{marvosym}
\usepackage{bbm}
\usepackage{fontawesome5}
\usepackage{siunitx}

\newtheorem{assumption}{Assumption}
\usepackage{thmtools}
\usepackage{thm-restate}

\usepackage{algorithmicx,algorithm}
\usepackage[noend]{algpseudocode}

\usepackage{placeins}

\usepackage[capitalize]{cleveref}
\Crefname{equation}{Eq.}{Eqs.}
\Crefname{figure}{Fig.}{Figs.}
\Crefname{tabular}{Tab.}{Tabs.}
\Crefname{remark}{Rem.}{Rems.}
\Crefname{section}{Sec.}{Secs.}
\Crefname{subsection}{Sec.}{Secs.}
\Crefname{proposition}{Prop.}{Props.}
\Crefname{example}{Ex.}{Exs.}

\DeclarePairedDelimiter{\delimabs}{\lvert}{\rvert}
\DeclarePairedDelimiter{\delimcardinality}{\lvert}{\rvert}
\DeclarePairedDelimiter{\delimnorm}{\lVert}{\rVert}

\NewDocumentCommand{\abs}{sm}{\IfBooleanTF{#1}{\delimabs*{#2}}{\delimabs{#2}}}
\NewDocumentCommand{\cardinality}{sm}{\IfBooleanTF{#1}{\delimcardinality*{#2}}{\delimcardinality{#2}}}
\NewDocumentCommand{\norm}{sm}{\IfBooleanTF{#1}{\delimnorm*{#2}}{\delimnorm{#2}}}
\NewDocumentCommand{\powerset}{r()}{2^{#1}}
\newcommand{\setcomplement}[1]{\overline{#1}}

\newcommand{\eqdef}{\vcentcolon=}

\newcommand{\UnionSym}{\bigcup}

\newcommand{\IntersectionSym}{\bigcap}

\newcommand{\Union}{\UnionSym}

\newcommand{\Intersection}{\IntersectionSym}

\NewDocumentCommand{\Measures}{d()}{\IfValueTF{#1}{\Pi(#1)}{\Pi}}
\NewDocumentCommand{\Distributions}{d()}{\IfValueTF{#1}{\mathcal{D}(#1)}{\mathcal{D}}}
\NewDocumentCommand{\integral}{d<> m m}{\IfValueTF{#1}{\int_{#1} #2\,d#3}{\int #2\,d#3}}
\NewDocumentCommand{\Expectation}{s d[]}{\IfValueTF{#2}{\mathbb{E}}{\mathbb{E}\IfBooleanTF{#1}{\left[#2\right]}{[#2]}}}
\NewDocumentCommand{\Probability}{s d[]}{\mathop{\mathrm{Pr}}\IfValueT{#2}{\IfBooleanTF{#1}{\left[#2\right]}{[#2]}}}

\newcommand{\MDP}{\mathcal{M}}

\newcommand{\States}{S}
\newcommand{\initialstate}{{\hat{s}}}
\NewDocumentCommand{\mctransitions}{d()}{\IfValueTF{#1}{P(#1)}{P}}

\newcommand{\Actions}{A}
\NewDocumentCommand{\stateactions}{r()}{{\Actions}(#1)}
\NewDocumentCommand{\mdptransitions}{d()}{\IfNoValueTF{#1}{\mathsf{P}}{\mathsf{P}(#1)}}

\newcommand{\infinitepath}{\rho}
\newcommand{\finitepath}{\varrho}
\NewDocumentCommand{\Infinitepaths}{d<>}{\IfValueTF{#1}{\mathsf{Paths}_{#1}}{\mathsf{Paths}}}
\NewDocumentCommand{\Finitepaths}{d<>}{\IfValueTF{#1}{\mathsf{FPaths}_{#1}}{\mathsf{FPaths}}}
\newcommand{\strategy}{\pi}
\NewDocumentCommand{\Strategies}{d<>}{\IfNoValueTF{#1}{\Pi}{\Pi_{#1}}}
\NewDocumentCommand{\StrategiesMD}{d<>}{\IfNoValueTF{#1}{\Pi}{\Pi_{#1}}^{\mathsf{MD}}}

\newcommand{\val}{\mathsf{V}}

\DeclareMathOperator{\SccsOp}{SCC}\NewDocumentCommand{\Sccs}{r()}{\SccsOp(#1)}
\DeclareMathOperator{\BsccsOp}{BSCC}\NewDocumentCommand{\Bsccs}{r()}{\BsccsOp(#1)}
\DeclareMathOperator{\EcsOp}{EC}\NewDocumentCommand{\Ecs}{d()}{\IfNoValueTF{#1}{\EcsOp}{\EcsOp(#1)}}
\DeclareMathOperator{\MecsOp}{MEC}\NewDocumentCommand{\Mecs}{d()}{\IfNoValueTF{#1}{\MecsOp}{\MecsOp(#1)}}

\NewDocumentCommand{\ProbabilityMC}{s r<> d[]}{\mathsf{Pr}_{#2}\IfNoValueF{#3}{\IfBooleanTF{#1}{\!\left[#3\right]\!}{[#3]}}}
\NewDocumentCommand{\ProbabilityMDP}{s r<> r<> d[]}{\mathsf{Pr}_{#2}^{#3}\IfNoValueF{#4}{\IfBooleanTF{#1}{\!\left[#4\right]\!}{[#4]}}}
\NewDocumentCommand{\ProbabilityMDPmax}{s r<> d[]}{\mathsf{Pr}_{#2}^{\max}\IfNoValueF{#3}{\IfBooleanTF{#1}{\!\left[#3\right]\!}{[#3]}}}
\NewDocumentCommand{\ProbabilityMDPsup}{s r<> d[]}{\mathsf{Pr}_{#2}^{\sup}\IfNoValueF{#3}{\IfBooleanTF{#1}{\!\left[#3\right]\!}{[#3]}}}

\NewDocumentCommand{\ExpectationMC}{s r<> r[]}{\mathbb{E}_{#2}\IfBooleanTF{#1}{\!\left[#3\right]\!}{[#3]}}

\newcommand{\Reachset}{T}
\NewDocumentCommand{\stepreach}{r<>}{\Diamond^{=#1}}
\NewDocumentCommand{\boundedreach}{r<>}{\Diamond^{{\leq}#1}}
\newcommand{\reach}{\Diamond}

\NewDocumentCommand{\steadystate}{d<> d()}{\IfValueTF{#1}{\pi^\infty_{#1}}{\pi^\infty}\IfValueT{#2}{(#2)}}

\usepackage{fontawesome5}

\newcommand{\stateSeq}{\mathcal{S}}

\newcommand{\smallsupport}{\texttt{Small Support}\xspace}
\newcommand{\independence}{\texttt{Independence}\xspace}
\newcommand{\equivalencestructures}{\texttt{Equivalence Structures}\xspace}

\newcommand{\fragments}{\texttt{Fragments}\xspace}
\newcommand{\chainfragments}{\texttt{Chain Fragments}\xspace}

\title{What Are the Odds? Improving Statistical Model Checking of Markov Decision Processes}
\vspace{-1cm}

\titlerunning{What Are the Odds? Improving the Foundations of SMC}

\author{Tobias Meggendorfer\inst{1}\orcidlink{0000-0002-1712-2165}
	\and
	Maximilian Weininger\inst{2,3}\orcidlink{0000-0002-1825-0097}
	\and
	Patrick Wienh\"oft\inst{4,5}\orcidlink{0000-0001-8047-4094}
}

\authorrunning{T.~Meggendorfer, M.~Weininger, P.~Wienh\"oft}

\institute{
	Lancaster University Leipzig, Leipzig, Germany
	\and
	Ruhr University Bochum, Bochum, Germany
	\and
	Institute of Science and Technology Austria, Klosterneuburg, Austria
	\and
	Dresden University of Technology, Dresden, Germany
	\and
	Centre for Tactile Internet with Human-in-the-Loop (CeTI), Dresden, Germany
	\email{tobias@meggendorfer.de}, \email{maximilian.weininger@rub.de}, \email{patrick.wienhoeft@tu-dresden.de}\vspace{-1em}}

\usepackage{paralist}

\crefname{section}{Sec.}{Secs.}
\crefname{appendix}{App.}{Apps.}
\crefname{lemma}{Lem.}{Lemms.}
\crefname{theorem}{Thm.}{Thms.}
\crefname{corollary}{Cor.}{Cors.}
\crefname{equation}{Eq.}{Eqs.}
\crefname{figure}{Fig.}{Figs.}
\crefname{table}{Tab.}{Tabs.}
\crefname{assumption}{Asm.}{Asms.}

\Crefname{section}{Sec.}{Secs.}
\Crefname{appendix}{App.}{Apps.}
\Crefname{lemma}{Lem.}{Lemms.}
\Crefname{theorem}{Thm.}{Thms.}
\Crefname{corollary}{Cor.}{Cors.}
\Crefname{equation}{Eq.}{Eqs.}
\Crefname{figure}{Fig.}{Figs.}
\Crefname{table}{Tab.}{Tabs.}
\Crefname{assumption}{Asm.}{Asms.}

\usetikzlibrary{colorbrewer}
\usepackage{array}
\usepackage{todonotes}
\usepackage{mdframed}

\tcbset{every box/.style={
	left=1mm,right=1mm,top=1pt,bottom=1pt,boxrule=1pt,colbacktitle=white!0,colback=white!0,coltitle=black
}}

\colorlet{linecolor}{gray}
\newtcolorbox{linebox}[1]{
	empty,
	left skip=1mm,
	attach boxed title to top left,
	minipage boxed title,
	title=#1,
	boxed title style={empty,size=minimal,toprule=0pt,top=1mm,left=2mm,bottom=1mm,overlay={}},
	coltitle=black,fonttitle=\bfseries\scshape,
	before=\par\medskip\noindent,parbox=false,boxsep=0pt,left=2mm,right=0mm,top=2pt,breakable,pad at break=0mm,
	before upper=\csname @totalleftmargin\endcsname0pt,
	overlay unbroken={\draw[linecolor,line width=2pt] ([xshift=-0pt]title.north west) -- ([xshift=-0pt]frame.south west); },
	overlay first={\draw[linecolor,line width=2pt] ([xshift=-0pt]title.north west) -- ([xshift=-0pt]frame.south west); },
	overlay middle={\draw[linecolor,line width=2pt] ([xshift=-0pt]frame.north west) -- ([xshift=-0pt]frame.south west); },
	overlay last={\draw[linecolor,line width=2pt] ([xshift=-0pt]frame.north west) -- ([xshift=-0pt]frame.south west); },
}

\begin{document}

\maketitle
\begin{abstract}
	Markov decision processes (MDPs) are a fundamental model of decision making which exhibit non-deterministic choice as well as probabilistic uncertainty.
	Traditionally, verification assumes exact knowledge of the probabilities that govern the behaviour of an MDP.
	However, this assumption often is unrealistic, e.g.\ when modelling cyber-physical systems or biological processes.
	There, we can employ statistical model checking (SMC) to obtain an estimate of the MDP's value (e.g.\ the maximal probability of reaching a goal state) that is close to the true value with high confidence (probably approximately correct).
	Model-based SMC algorithms sample the MDP and build a model of it by estimating all transition probabilities, essentially for every transition answering the question: \enquote{What are the odds?}
	However, so far the statistical methods employed by state-of-the-art SMC verification algorithms are quite naive or even compromise the correctness guarantees.

	Our first contribution is to categorize and analyse statistical methods, identifying those that are most efficient and that provide suitable guarantees for the verification setting.
	As our second contribution, we propose improvements that exploit structural knowledge of the MDP.
	Both contributions generalize to many types of problem statements as they are largely independent of the setting.
	Moreover, our experimental evaluation shows that they lead to significant gains, reducing the number of samples that an SMC algorithm has to collect by up to two orders of magnitude.


\end{abstract}

\section{Introduction}\label{sec:intro}
\emph{Markov decision processes} (MDPs) \cite{Puterman} are \emph{the} classic modelling formalism for dynamic systems with probabilistic and nondeterministic behaviour.
In essence, MDPs comprise several states and each state has an associated set of available actions to choose from.
In order to capture the \emph{aleatoric uncertainty} (the randomness of the process, e.g.\ a coin toss), each action corresponds to a distribution on the successor states rather than a single successor.
The system evolves from a state by choosing an action, moving to a successor sampled from the corresponding distribution, and repeating this process ad infinitum.

\emph{Infinite-horizon reachability} asks: What is the \emph{optimal} probability to eventually reach a given set of target states?
Solving an MDP means to correctly computing this \emph{value} where the optimum is taken over the set of all \emph{strategies}, i.e.\ ways to choose actions in every state.

\paragraph{Restricted knowledge.}
Traditionally, verification procedures assume full knowledge of the MDP.
However, this is often unrealistic in practice, as for example probabilities governing biological processes or cyber-physical systems are usually not known precisely, see also~\cite[Chp.~28.7.3]{handbook} or~\cite{DBLP:journals/sttt/BadingsSSJ23}.
MDPs with unknown transition distributions exhibit \emph{epistemic uncertainty}~\cite{DBLP:journals/sttt/BadingsSSJ23} (the lack of knowledge, e.g.\ unknown bias of a coin).
Clearly, in this case we cannot give exact results, as we do not know the exact system we are dealing with.
Nevertheless, the desire to obtain guarantees on the correctness of the result remains unchanged.

\paragraph{PAC guarantees.}
With unknown transition probabilities, the goal of verification usually is to guarantee that the result is \emph{probably approximately correct} (PAC):
Given a confidence budget $\delta$ and a precision $\varepsilon$, we want a result that is \emph{probably} (with probability greater than $1-\delta$) \emph{approximately correct} (at least $\varepsilon$-precise).
For example, consider a coin with unknown bias.
If we sample it 1000 times and see heads 800 times, its bias is likely around 0.8.
More formally, for a required confidence of $95\%$ (i.e.\ $\delta=0.05$), using a statistical method such as Hoeffding's inequality~\cite{Hoe63} (see \Cref{sec:3-coin}) yields that the coin's bias is likely within $0.8 \pm 0.04$.

\paragraph{Statistical model checking: Markov chains.}
Markov chains effectively are MDPs where every state has exactly one available action.
As such, the outcome is purely stochastic (without any non-determinism) and these systems describe a single random variable.
Here, obtaining PAC guarantees is straightforward.
Intuitively, we gather simulations of the Markov chain and remember whether they reach a target state or not.
Then, we can estimate this probability of reaching the target analogous to the above coin example.
We refer to extensive surveys of statistical model checking for Markov chains with properties expressed in linear temporal or computational tree logic~\cite{DBLP:conf/isola/Kretinsky16,DBLP:journals/tomacs/AghaP18,DBLP:series/lncs/LegayLTYSG19}, or reward properties~\cite{WatS}.
Notably, in the context of Markov chains, statistical approaches may even be preferred over traditional methods to combat state space explosion, as sampling may still be feasible even when the model is too large to be analysed precisely.
However, for MDPs, state-of-the-art statistical approaches are far less efficient than traditional methods; and the main interest is to obtain PAC guarantees in the presence of epistemic uncertainty, where traditional methods are simply not applicable.

\paragraph{Statistical model checking: MDPs.}
Solving MDPs with unknown transition probabilities is fundamentally more difficult than solving purely stochastic Markov chains.
This is because the satisfaction of the property depends on the \emph{strategy}; already for reachability, a naive approach would have to check exponentially many of these.
One way of tackling this problem~\cite{atva14,DBLP:conf/rss/FuT14,DBLP:journals/tac/WenT22,THEORETICS} are \emph{model-free} approaches (see~\cite[Rem.~5.1]{THEORETICS} for details).
These algorithms require an astronomical number of samples to provide non-trivial bounds~\cite{atva14,DBLP:journals/tac/WenT22,THEORETICS} (see~\cite[Sec.~4]{AKW19}) or knowledge of the mixing time of the MDP~\cite{DBLP:conf/rss/FuT14}, which is as hard to compute as the value.
Thus, state-of-the-art statistical verification of MDPs is \emph{model-based}.
Essentially, it proceeds by first learning the unknown transition probabilities in the MDP, thereby constructing an MDP with full knowledge that is probably correct, and then solving this using traditional methods.
This way, PAC-guarantees can be obtained efficiently in many settings, including MDPs that are com\-mu\-ni\-cating~\cite{AueJakOrt08}, continuous-time~\cite{agarwal2022pac}, or  continuous-space~\cite{DBLP:journals/jair/BadingsRAPPSJ23}, in stochastic games~\cite{AKW19}, and for properties that are $\omega$-regular \cite{DBLP:journals/corr/abs-2310-12248} or consist of multiple objectives~\cite{WeiningerGMK21}.

\paragraph{Sample efficiency.}
To gather samples, we need to simulate or interact with the system, which may require significant effort (e.g.\ for cyber-physical systems).
Naturally, we thus are interested in obtaining PAC bounds efficiently, i.e.\ with as few samples as possible.
For Markov chains, we refer again to~\cite{DBLP:conf/isola/Kretinsky16,DBLP:journals/tomacs/AghaP18,DBLP:series/lncs/LegayLTYSG19,WatS}, and in particular highlight the efforts for dealing with rare events~\cite{DBLP:reference/npe/Villen-AltamiranoV11,DBLP:conf/setta/BuddeDH17}.
For MDPs, many works focus on simpler cases of finite-horizon or discounted properties, bounding the sample complexity in terms of the horizon or discount factor, e.g.~\cite{LP12,DBLP:conf/colt/AgarwalKY20,DBLP:conf/icml/JinKSY20,DBLP:conf/aaai/HasanzadeZonuzy21,DBLP:conf/icml/Shi0W0C22,DBLP:conf/colt/WagenmakerSJ22}.
Notably, early works contained subtle mistakes compromising guarantees, see \cite[Sec.~3.3]{DBLP:conf/isola/Kretinsky16}.
Our focus in this work is on \emph{infinite-horizon} objectives, focussing on reachability for simplicity, and detailing in \Cref{app:other-objectives} how our contributions extend to other objectives.

\paragraph{Our contribution} is practically improving the sample efficiency of model-based statistical verification for MDPs with unknown transition probabilities.
Colloquially speaking, when learning the transition probabilities, we provide a \emph{more precise} answer to the question: \enquote{What are the odds?}.
The methods we propose make the most of the data, i.e.\ given the same simulations of the system, they provide better probability estimates than the state-of-the-art and as such also transparently improve the sample efficiency of these approaches.
Concretely, our contribution is threefold:

\emph{Firstly}, we discuss estimation of categorical distributions, the central component of model-based approaches
(\Cref{sec:3-title}).
To this end, we survey relevant literature and contrast numerous statistical methods.
We identify most inequalities as being unsuitable due to insufficient guarantees, as well as several SMC approaches that claim correctness despite using them.
Moreover, we prove a statement of independent interest about the maximum size of the confidence interval produced by Hoeffding's inequality and the Clopper-Pearson interval (\cref{lemma:sample-complexity}), solving an open question from~\cite{DBLP:journals/tse/BuS24}.
We empirically show that the former produces intervals significantly greater than the latter (\cref{sec:3-sample-number}).
Thus, for the purpose of estimating transition probabilities, the Clopper-Pearson interval is always preferable.

\emph{Secondly}, we provide techniques for exploiting the knowledge we have about the MDP and the property of interest (\Cref{sec:4-title}).
They allow us to invest less or even no part of our confidence budget $\delta$ for certain state-action pairs.
Many of them are based on already known observations that however -- to our knowledge -- have not yet been applied in this context.
Moreover, we propose the novel technique of collapsing so-called \fragments.

\emph{Thirdly}, we empirically evaluate the impact of our improvements (\Cref{sec:5-title}).
We conclude that our methods always have a positive impact: They have practically no computational overhead and always reduce the number of samples necessary to achieve a given precision $\varepsilon$, in many cases by two orders of magnitude.

\paragraph{Relevance.}
Regarding \emph{efficiency}, a collection of statistical methods and structural improvements is overdue:
Most papers in the verification community overlook their potential, even those with the declared goal of improving the scalability and practical applicability of statistical approaches~\cite{AKW19,agarwal2022pac,DBLP:journals/jair/BadingsRAPPSJ23}.
In particular, the state-of-the-art method for estimating categorical distributions is Hoeffding's inequality~\cite{AKW19,WeiningerGMK21,agarwal2022pac,BaiDubWie23,DBLP:journals/corr/abs-2310-12248}.
As mentioned, we empirically show that this is significantly worse than the (well-known in statistics) Clopper-Pearson interval (\Cref{sec:3-sample-number}).
Further, we prove that the method for estimating distributions developed in~\cite{DBLP:journals/jair/BadingsRAPPSJ23} is in fact a weaker version of the Clopper-Pearson interval (\Cref{app:jansen}).
Moreover, even arguably trivial ways of exploiting the structure (\Cref{sec:4-title}) such as \smallsupport and \independence have not been employed, let alone the more advanced techniques of \equivalencestructures and \fragments.

With respect to \emph{soundness}, we show that several papers use methods that \emph{compromise the PAC-guarantees}~\cite{Arnd,BazilleGJS20,DBLP:conf/nips/SuilenS0022} (see \Cref{app:other-conf-intervals}).
Our survey clearly separates methods for estimating distributions that are unsuitable for our setting, and we provide proofs for the soundness of all structural improvements.

Finally, while our focus is on tackling MDPs, many of our findings are also applicable in purely probabilistic cases, such as DTMCs.
In particular, building on the preprint of our work \cite{arxiv}, the recent work \cite[Tab.~1]{WatS} shows that almost all tools that apply statistical methods in the context of DMTCs employ methods that are inefficient or unsound.

\section{Preliminaries}

\noindent
A \emph{probability distribution} over a countable set $X$ is a mapping $d : X \to [0,1]$, such that $\sum_{x \in X} d(x) = 1$.
The set of all probability distributions on $X$ is $\Distributions(X)$.

\subsection{Markov Decision Processes}

\noindent
A \emph{Markov decision process (MDP)}, e.g.\ \cite{Puterman}, is a tuple $\MDP = (\States, \Actions, \mdptransitions)$, where
$\States$ is a finite set of states;
$\Actions$ is a finite set of actions, overloaded to yield for each state $s \in \States$ a non-empty set of \emph{available actions} $\stateactions(s) \subseteq \Actions$;
and $\mdptransitions : \States \times \Actions \to \Distributions(\States)$ is the (partial) transition function, that yields for each state $s \in \States$ and $a \in \stateactions(s)$ the associated distribution over successor states $\mdptransitions(s, a)$.
For ease of notation, we write $\mdptransitions(s, a, s')$ instead of $\mdptransitions(s, a)(s')$.

The \emph{semantics} of MDPs is defined in the usual way by means of paths, strategies and the probability measure in the induced Markov chain.
We briefly recall this here and refer to~\cite[Chp.~10]{DBLP:books/daglib/0020348} for an extensive introduction.
Let $\MDP$ be an MDP.
An \emph{infinite path} is a sequence of state-action pairs $\infinitepath = s_1 a_1 s_2 a_2 \cdots \in (\States \times \Actions)^\omega$ with $\mdptransitions(s_i, a_i, s_{i+1}) > 0$.
We denote by $\infinitepath(i)$
the $i$-th state $s_i$ in a given path and by $\Infinitepaths<\MDP>$ the set of all infinite paths.
A (memoryless deterministic, MD) \emph{strategy} is a mapping $\strategy : \States \to \Actions$, choosing one enabled action in each state, i.e.\ $\strategy(s) \in \Actions(s)$.
We write $\StrategiesMD<\MDP>$ to refer to all MD strategies.
Complementing an MDP with such a strategy and an initial state $\initialstate \in \States$ yields a Markov chain that induces a unique probability measure $\ProbabilityMDP<\MDP, \initialstate><\strategy>$ over infinite paths \cite[Chp.~10.1]{DBLP:books/daglib/0020348}.

An \emph{objective} formalises the goal of the MDP.
For simplicity, we focus on \emph{reachability}, and in \Cref{app:other-objectives} explain how our methods extend to other objectives.
Define $\reach \Reachset \coloneqq \{\infinitepath \in \Infinitepaths<\MDP> \mid \exists i.~\infinitepath(i) \in \Reachset\}$ as the set of paths that eventually reach $\Reachset$.
The \emph{value} of a state in an MDP is the maximum probability to achieve the objective, i.e.\ reach the goal states, under any strategy.
Formally, the value of state $s$ is defined as $\val_{\MDP}(s) \eqdef \max_{\strategy \in \StrategiesMD<\MDP>} \ProbabilityMDP<\MDP, s><\strategy>[\reach \Reachset]$.
Note that MD strategies are sufficient to maximise the reachability probability~\cite[Lem.\ 10.102]{DBLP:books/daglib/0020348}.

\subsection{Statistical Guarantees and Statistical Model Checking}\label{sec:2-smc}
\noindent
In this work, we deal with MDPs where the transition function $\mdptransitions$ is unknown.
There are multiple approaches to tackle this problem, and we focus on \emph{model-based statistical model checking} (SMC) (see \Cref{sec:intro,sec:RW} for discussion of other SMC approaches).
Algorithms for model-based SMC, e.g.~\cite{AKW19,WeiningerGMK21,agarwal2022pac,DBLP:conf/nips/SuilenS0022,BaiDubWie23,DBLP:journals/jair/BadingsRAPPSJ23}, comprise three conceptual steps:
First, they obtain a finite number of samples from the MDP (see \cref{sec:RW} for a discussion of sampling strategies).
Then, from these samples they construct confidence intervals on each \emph{transition} probability.
Finally, they solve the induced interval MDP~\cite{givan2000bounded}, yielding bounds on the true value.

Our work focuses on the second part of model-based SMC:
Given a finite number of samples, what is the best way to estimate transition probabilities?
Colloquially, we \enquote{make the most of the data} by obtaining \enquote{as small as possible} intervals.\footnote{We intentionally do not ask for the \emph{minimum} interval size, because it is a random variable (as it depends on the sampling outcome) and minimising its expected value would require assuming a prior distribution over $P$, which we cannot justify.}
By getting more precise estimates from the samples, we reduce the width of the confidence intervals on transition probabilities in the second phase of the SMC algorithms, and thus also improve their overall performance.
We additionally assume knowledge of the topology, i.e.\ that we know the support $\{s' \mid \mdptransitions(s,a,s')>0\}$ of every distribution $\mdptransitions(s,a)$, also called \emph{grey-box setting}~\cite{AKW19,agarwal2022pac}, a standard assumption for SMC.
This is relevant for our structural improvements in \Cref{sec:4-title}.
\Cref{sec:black-box} discusses how our methods extend to the completely opaque \emph{black-box setting}~\cite{tacas16,AKW19}.
Together, we formalize our problem as follows.

\begin{linebox}{Problem Statement: Grey-Box SMC of MDPs}
	\textbf{Input:} A confidence budget $\delta$, an MDP $\MDP$ with unknown $\mdptransitions$, but known support of each distribution $\mdptransitions(s, a)$, and a sequence of random samples $(s,a,s') \in \States\times\Actions\times\States$ from $\MDP$.

	\noindent
	\textbf{Output:} An interval $[\underline{v},\overline{v}]$ such that $\underline{v} \leq \val_{\MDP}(\initialstate) \leq\overline{v}$ with probability at least $1-\delta$ and $\overline{v}-\underline{v}$ is as small as possible.
\end{linebox}

\newcommand{\eventAllCorrect}{\text{Corr}}
\newcommand{\eventTransCorrect}[1]{\text{Corr}(#1)}

\paragraph{State-of-the-art} methods~\cite{AKW19,WeiningerGMK21,agarwal2022pac,DBLP:conf/nips/SuilenS0022,DBLP:journals/jair/BadingsRAPPSJ23,BaiDubWie23} usually
distribute the confidence budget $\delta$ uniformly over all transitions and apply \emph{Hoeffding's inequality} (see \Cref{sec:3-coin}) to get a confidence interval for each of them.
Every transition then is correct with probability at least $1 - \frac \delta {\abs{\mdptransitions}}$, where $\abs{\mdptransitions}$ is the number of transitions.
By union bound, the probability of all transition being correct is greater than $1-\delta$.
In \Cref{app:other-conf-intervals}, we comment on the few cases in verification literature where methods other than Hoeffding's inequality are employed for model-based SMC.
These either compromise the PAC-guarantee~\cite{BazilleGJS20,DBLP:conf/nips/SuilenS0022} or we prove their inferiority~\cite{DBLP:journals/jair/BadingsRAPPSJ23}.

\section{Statistical Methods for Estimating Probabilities}\label{sec:3-title}
\noindent
As mentioned, estimating transition probabilities lies at the heart of model-based SMC.
This section presents methods to estimate the distribution $\mdptransitions(s,a)$ for a single state-action pair $(s,a)$ in the given MDP.
\Cref{sec:3-die} discusses estimating the whole distribution (\enquote{a die}) at once.
\Cref{sec:3-coin} details how to estimate a single transition probability $\mdptransitions(s,a,s')$; essentially viewing the samples drawn from $\mdptransitions(s,a)$ as a coin toss, either reaching $s'$ or not.
We survey the literature, including methods used in verification, the concentration inequalities from~\cite{conc-ineq} applicable in our setting, and methods from statistics literature~\cite{agresti1998approximate,BroCaiDas01,hazardous2002calculating,DBLP:journals/jql/Wallis13,dean2015evaluating}.
Finally, \Cref{sec:3-sample-number} complements our survey with a theoretical and experimental efficiency analysis of the most relevant methods.

\subsection{Estimating a Die}\label{sec:3-die}
The most common method to estimate the distribution $P(s,a)$ as a whole is based on~\cite{WeiOrdSer03}, and used by, e.g., MBIE~\cite{StrLit04}, and UCRL2 \cite{AueJakOrt08}.
It computes a maximum likelihood estimate of the probability distribution, i.e.\ the empirical average of each outcome, and then constructs the confidence region as an $L^1$-ball around this estimate.
In other words, the confidence region contains all probability distributions whose $L^1$-distance from the empirical average is less than a certain value which depends on the number of samples and the number of successors~\cite{WeiOrdSer03}.

The advantage of the $L^1$-ball is that it accounts for the dependence between the transition probabilities, i.e.\ that $\sum_{i=1}^k p_i = 1$, making it seem like a canonical choice.
However, the method scales poorly with the number of successors of a state-action pair~\cite{WieSuiSim23}.
One can transform the model such that every state-action pair has two successors, minimising the number of samples required to prove a property~\cite[App.~A]{WieSuiSim23a}.
However, then the method of~\cite{WeiOrdSer03} coincides with applying Hoeffding's inequality to each transition probability individually (see \Cref{sec:app-l1-is-hoeffding}) which, as we show in \Cref{sec:3-sample-number}, is a sub-optimal way to estimate transition probabilities.
Thus, we focus on the approach that estimates every transition probability individually.

\subsection{Estimating a Coin}\label{sec:3-coin}
\noindent This section discusses the most basic SMC-problem: estimating a single transition probability $\mdptransitions(s,a,s')$ with confidence budget $\delta$.
The relevant information from the samples are (i)~how often was $(s,a)$ sampled and (ii)~how often was the successor $s'$ chosen, effectively a binomial distribution (sampling $s'$ is a success).
Throughout this section, we refer to a binomial distribution with success probability $p$ (also called \enquote{binomial proportion}) and a test sequence on it with $n$ trials and $k$ successes.
We write $\hat{p}=\frac{k}{n}$ for the maximum likelihood estimate of $p$.

\begin{linebox}{Probability Estimation Problem}
	\textbf{Input:} A confidence budget $\delta$ and $n$ random samples drawn from a binomial distribution with \emph{unknown} success probability $p$.\\
	\textbf{Output:}
	A confidence interval $[\underline{p},\overline{p}]$ for which $\Probability[\underline{p} \leq p \leq \overline{p}] \geq 1 - \delta$.
\end{linebox}
\begin{remark}\label{rem:asymm}
	Many works focus on confidence intervals centred around $\hat{p}$ (either additive $[\hat{p} - \varepsilon, \hat{p} + \varepsilon]$ or relative $[\hat{p} \cdot (1 - \varepsilon), \hat{p} \cdot (1+\varepsilon)]$).
	However, this excludes potentially tighter \emph{asymmetrical} confidence intervals.
\end{remark}
For soundness, we require that the \emph{coverage probability} $\Probability[\underline{p} \leq p \leq \overline{p}]$ is consistently at least $1-\delta$ \emph{for all $p$}.
(Recall that $p$ is fixed but unknown and the random variable is $\hat p$.)
Notably, this is different from requiring \emph{average coverage}, i.e.\ if $p$ were chosen uniformly at random (corresponding to a Bayesian approach assuming a uniform prior).
The explicit goal of SMC is to achieve correctness guarantees for \emph{all} models, without any prior assumptions.
Also, note that the problem statement only requires correctness, allowing trivial solutions such as $[0,1]$; we discuss minimizing the interval size in \cref{sec:3-sample-number}.
We now present two methods solving the Probability Estimation Problem.

\paragraph{Hoeffding's inequality.}\label{subsec:hoeffding}
Hoeffding's seminal paper \cite{Hoe63} provides a confidence interval for the sum of random variables. Several works have applied this to estimate the mean of binomial random variable $X$, i.e.\ $X\in\{0,1\}$.
Applying Hoeffding's inequality yields the $(1-\delta)$ confidence interval $[\underline{p}_{ho},\overline{p}_{ho}]$ where  $\underline{p}_{ho}=\max\{0,\hat{p}-c_{ho}\}$, $\overline{p}_{ho} = \min\{1,\hat{p}+c_{ho}\}$, and $c_{ho}=\sqrt{\ln(2/\delta) / 2n}$ \cite{AKW19-arxiv}.

This result is often referred to as Hoeffding's inequality, although Okamoto considered the special case of
binomial random variables before~\cite{Okamoto59}.
Further, Hoeffding's is a specialization of Chernoff's and Markov's inequality~\cite[Chp. 2]{conc-ineq}.
Thus, variants of this bound can appear under some combinations of these names, e.g.\ Okamoto-Chernoff inequality in \cite{BazilleGJS20}.
We call it Hoeffding's inequality, as that is common in verification literature.
While the number of samples required by Hoeffding's inequality for a fixed $\varepsilon$ is asymptotically optimal at $\mathcal{O}(1/\varepsilon^2)$ (see, e.g.\ \cite{Bartlett1999}), we show in \cref{sec:3-sample-number} that in practice we can do significantly better.

\paragraph{Clopper-Pearson interval.}\label{subsec:cp}
A widely used confidence interval method was introduced by Clopper and Pearson \cite{CloPea34}, sometimes called the ``exact method''.
The approach inverts the task and, given $n$ and $k$, asks for the minimum (and maximum) $p$ such that observing at least (at most) $k$ out of $n$ successes has a probability of $1-\frac{\delta}{2}$.
The $(1-\delta)$ confidence interval using the Clopper-Pearson interval can be represented in a closed form using the inverse regularised beta function $I^{-1}_x(a,b)$ \cite{Tem92} as
$
	[\underline{p}_{cp}, \overline{p}_{cp}] = [I^{-1}_{\delta/2}(k+1,n-k), I^{-1}_{1-\delta/2}(k,n-k+1)]
$
for $0<k<n$.
Additionally, for $k=0$ (or $k=n$) the lower (or upper) bound needs to be taken as 0 (or 1).
Note that, unlike Hoeffding's inequality, it is generally not centered around $\hat{p}$ but has its center shifted towards $\frac{1}{2}$.

Recently, two methods related to Clopper-Pearson appeared:
a newly developed approach~\cite{DBLP:journals/jair/BadingsRAPPSJ23}
that we show to be a strictly weaker version of the Clopper-Pearson interval, and a sequential variant~\cite{DBLP:journals/tse/BuS24}
; we provide details in~\Cref{app:other-conf-intervals}.

\begin{theorem}[From~{\cite[App. D]{AKW19-arxiv}} and~\cite{CloPea34}]\label{thm:coin}
	Hoeffding's inequality and the Clopper-Pearson interval both solve the Probability Estimation Problem.
\end{theorem}

\paragraph{Further confidence methods.}
We now briefly discuss several other confidence interval methods and reasons why we do not recommend to use them in the context of SMC with PAC-guarantees for Markov systems.

Many statistical methods either fail to guarantee any coverage probability at all, or only provide an \emph{average} coverage probability of $1-\delta$ when $p$ is uniform over $[0,1]$.
These include
the (adjusted) central limit theorem~\cite{hazardous2002calculating}, Wald interval \cite{BroCaiDas01,DBLP:journals/jql/Wallis13}, the Wilson score interval \cite{Wil27,BroCaiDas01,DBLP:journals/jql/Wallis13}, the Agresti-Coull interval~\cite{agresti1998approximate}, the Arcsine interval \cite{BroCaiDas01}, and the Logit interval \cite{BroCaiDas01}.
For the Wilson score, Newcombe introduced a continuity corrected version with better coverage properties \cite{New98}.
However, even with the continuity correction, the coverage is insufficient.
We discuss this in more detail in \Cref{app:other-conf-intervals}.

Hypothesis tests, such as Student's $t$-test \cite{student1908probable} or the sequential probability ratio test (SPRT) \cite{sprt} are designed to distinguish between two hypotheses.
As such, in the context of SMC they are widely used for answering threshold queries (see e.g.\ \cite{DBLP:conf/rv/LegayDB10,DBLP:journals/sttt/ReijsbergenBSH15}), but they are not suitable for constructing confidence intervals.

Monte Carlo simulations, e.g.~\cite{dklr,DBLP:journals/corr/Huber13a,DBLP:journals/corr/abs-2210-12861}, provide an efficient $\varepsilon$-approximation for the mean of Bernoulli random variables.
However, they require $\delta$ and $\varepsilon$ to be given and then define a stopping condition which needs unlimited sampling access for every state-action pair.
This does not fit the SMC algorithm structure:
As the impact of a transition probability on the the value is unclear, we cannot fix the precision for individual transitions a-priori.

Jeffrey's interval \cite{BroCaiDas01} gives guarantees in a Bayesian sense by providing a \emph{credible interval}.
We avoid credible intervals as they rely on a prior distribution over $p$ which we cannot justify in our setting.
However, such Bayesian methods may be considered if prior distributions are known.

Several methods provide proper coverage probability, but are either not applicable to our setting or provably suboptimal:
The Massart bound improves Hoeffding's inequality for $p\neq \frac{1}{2}$ \cite{DBLP:journals/tomacs/JegourelSD19}, but we do not know whether this assumption is satisfied for a given transition.
Similarly, Bennett's inequality \cite{Ben62} improves Hoeffding's inequality by including information about the variance of the random variable which is unknown in our setting.
Using the trivial upper bound of $\frac 1 4$, Benett's inequality is always worse than Hoeffding's inequality (see \Cref{sec:app-bennett}).
More advanced methods estimating the variance from data asymptotically outperform Hoeffding's inequality \cite{MauPon09}, but only for $p$ away from $\frac{1}{2}$, and very large~$n$.
Bernstein's inequality \cite{Ber24} is a conservative relaxation of Bennett's inequality, which however provides strictly wider confidence intervals.
The Dvoretzky-Kiefer-Wolfowitz-Massart inequality (DKW) \cite{dkw,Massart90} constructs a \emph{confidence band} around the empirical distribution from which bounds on the mean can be derived.
However, since all sample values are extreme values in the binary case (i.e. $0$ or $1$), DKW coincides with Hoeffding's inequality when estimating probabilities \cite{WatS}.

Lastly, we summarize the discussion of~\cite{bcd-comment-santner,bcd-comment-corcoran} about further methods with proper coverage probability: The Blyth-Still(-Casella) interval \cite{BlySti83,CasMcC84}, inverted exact likelihood ratio test interval \cite{AitAndFra+05} and Duffy-Santner interval \cite{DufSan87}
all deliver intervals similar to the Clopper-Pearson interval; they are only slightly smaller for $p$ close to $0$ or $1$ in exchange for slightly larger intervals if $p\approx\frac{1}{2}$.
Since they are computationally expensive and, unlike the Clopper-Pearson interval, not readily available in a lot of statistical libraries
, we do not consider them in this paper.

\subsection{Sample Complexity}\label{sec:3-sample-number}
\noindent
In this section, we compare the two methods from \Cref{thm:coin} from a theory perspective with respect to their worst-case sample complexity:
We show how to obtain an a-priori bound on the number of samples required to obtain a fixed precision $\varepsilon$ (i.e.\ a confidence interval of width at most $\varepsilon$) for a single transition.
Empirical confirmation is presented in \cref{sec:5-title} in the context of learning full MDPs.

\begin{linebox}{Probability Sample Complexity Problem}
	\textbf{Input:} Confidence budget $\delta$, precision $\varepsilon$, and \emph{unknown} binomial parameter $p$.\\
	\textbf{Output:}
	The minimum number of samples $n$ to guarantee\setlength{\abovedisplayskip}{3pt}
	\begin{equation*}
		\Probability[\underline{p} \leq p \leq \overline{p} \mid \hat{p}=\tfrac{k}{n}] \geq 1 - \delta \quad \text{where} \quad \overline{p}-\underline{p} \leq \varepsilon
	\end{equation*}
\end{linebox}
\noindent
At first, this problem seems difficult to solve, as the interval width not only depends on the sample size $n$ and confidence budget $\delta$, but also on the empirical success rate $\hat{p}$.
However, we show in \Cref{app:proof-worst-case-sample} that both Hoeffding's inequality and the Clopper-Pearson interval maximise the interval size when $\hat{p}=\frac{1}{2}$.
\begin{restatable}{proposition}{samplecomplexity}
\label{lemma:sample-complexity}
Given a confidence budget $\delta \in (0,1]$, the confidence interval $[\underline{p},\overline{p}]$ computed by Hoeffding's inequality or by the Clopper-Pearson interval from a sequence $\stateSeq$ of $n$ samples from a binomial distribution maximises $\overline{p}-\underline{p}$ when $\stateSeq$ contains equally many positive and negative samples.
\end{restatable}
\noindent
\paragraph{Consequences of Proposition~\ref{lemma:sample-complexity}.}
This result is conjectured in~\cite[Hypoth.\ 1]{DBLP:journals/tse/BuS24} and required for completing their soundness proof~\cite[Thm.\ 1]{DBLP:journals/tse/BuS24}.
Further, \Cref{lemma:sample-complexity} is the basis for providing sequential SMC algorithms in~\cite[Sec.\ 3.2]{WatS}.

\begin{figure}[t]
	\centering

	\begin{minipage}{0.5\textwidth}\centering
		\begin{tikzpicture}
			\begin{axis}[
				width=\textwidth,height=3cm,
				table/col sep=comma,
				x label style={anchor=north,inner sep=0pt},
				y label style={anchor=south,inner sep=0pt},
				xmin=0.001,ymin=1,xmax=0.9,ymax=2.5,
				axis x line*=bottom,
				axis y line*=left,
				xmode=log,
				xlabel={\small Precision $\varepsilon$},ylabel={\small Ratio},
				]
				\addplot+[no marks,Dark2-B,thick] table[x index=0, y index=1] {data/ratio_eps_delta0.01.csv};
			\end{axis}
		\end{tikzpicture}
	\end{minipage}
	\begin{minipage}{0.5\textwidth}\centering
		\begin{tikzpicture}
			\begin{axis}[
				width=\textwidth,height=3cm,
				table/col sep=comma,
				x label style={anchor=north,inner sep=0pt},
				y label style={anchor=south,inner sep=0pt},
				xmin=0,ymin=1,xmax=1,ymax=60,
				axis x line*=bottom,
				axis y line*=left,
				xlabel={\small Probability $\hat{p}$},ylabel={\small Ratio},
				ymode=log
				]
				\addplot+[no marks,Dark2-B,thick] table[x index=0, y index=1] {data/ratio_phat_delta0.01.csv};
			\end{axis}
		\end{tikzpicture}
	\end{minipage}

	\caption{
		Left: Ratio of worst-case sample complexity ($\hat{p}=0.5$) between Hoeffding bound and Clopper-Pearson interval for confidence $\delta=0.01$ and varying precision $\varepsilon$.
		Right: Ratio of sample complexity between Hoeffding bound and Clopper-Pearson interval for varying $\hat{p}$, precision $\varepsilon=0.01$, and confidence $\delta=0.01$.
		Note the logarithmic scale for the X-axis on the left and Y-axis on the right.
	} \label{fig:sample-complexity}
\end{figure}

\paragraph{Emprical analysis.} We complement this result with an experimental analysis.
We solve the Probability Sample Complexity Problem by considering the worst-case $\hat{p}=\frac{1}{2}$ and computing the number of required samples by binary search on $n$ for varying precisions $\varepsilon$.
\Cref{fig:sample-complexity} (left) shows the \emph{ratio} of required samples between using Hoeffding's inequality and Clopper-Pearson.
The ratio rather consistently is at $\approx 1.5$ (or even larger for smaller $\delta$, see \Cref{sec:app-sample-delta}), indicating that the Clopper-Pearson interval only requires around two thirds of the samples that Hoeffding's inequality needs to guarantee the same precision for $\hat{p}=\frac{1}{2}$.

\noindent
While the worst-case sample complexity is the only sound a-priori bound, for practical purposes the sample complexity when $\hat{p}\neq\frac{1}{2}$ is also relevant.
Thus, in \Cref{fig:sample-complexity} (right) we instead fix a precision of $\varepsilon=0.01$ and confidence budget $\delta=0.1$ and now vary $\hat{p}$.
We again consider the ratio between the samples required to obtain a confidence interval of width $\leq \varepsilon$ using Hoeffding or Clopper-Pearson.
When $\hat{p}$ deviates from $\frac{1}{2}$, the advantage of Clopper-Pearson interval becomes even larger, often requiring an order of magnitude less samples, especially for $\hat{p}$ close to $0$ or $1$.
This is not surprising, since the width using Hoeffding's inequality is largely independent of $\hat{p}$ (except when $\hat{p}\pm\varepsilon\not\in [0,1])$) whereas $\hat{p}$ directly impacts the width of the Clopper-Pearson interval.
This also is relevant for our application, as transition probabilities in MDPs often are away from $\frac{1}{2}$.

\begin{linebox}{Key Takeaway}
	Algorithms estimating binomial probabilities (as in standard model-based SMC approaches) should refrain from using the commonly employed Hoeffding's inequality, as it is outperformed by the Clopper-Pearson interval, most notably for probabilities close to 0 or 1.
\end{linebox}

\section{Structural Improvements}\label{sec:4-title}
\noindent
After investigating the basic problem of estimating single transitions, we now turn to utilizing the structure of the MDP and property of interest.
State-of-the-art SMC algorithms naively distribute the confidence budget among all transitions and estimate the probability for each of them.
However, especially in the grey-box setting, we actually have a lot of structural information that we can utilise to improve estimates or even conclude that estimation is not needed at all.
In \Cref{sec:4-use-model}, we explain how to use information about the model structure, and
in \Cref{sec:4-use-prop}, we exploit the additional information about the property.

\begin{remark}[Applicability in the black-box Setting]\label{sec:black-box}
In the black-box setting~(as defined in \cite{tacas16,AKW19,THEORETICS}), we do not know the support of any distribution,
but instead only have a lower bound on the minimum occurring transition probability $p_{\min}$.
However, even given only that information, we can infer the topology using the methods from~\cite{tacas16,AKW19}:
Assume we have obtained confidence intervals on the transition probabilities of all successor states in a distribution $\mdptransitions(s,a)$ observed in our samples.
Intuitively, when the sum of all lower bounds of these confidence intervals is greater than $1-p_{\min}$, the probability of having overlooked a successor is less than our confidence budget.
Thus, from that point onward, we know the support of $\mdptransitions(s,a)$ with sufficient confidence and can apply structural improvements.
Using the Clopper-Pearson interval (\Cref{sec:3-coin}) assists in this, since smaller confidence intervals lead to knowing the topology sooner.
\end{remark}

\subsection{Using Information About the Model}\label{sec:4-use-model}
\paragraph{\smallsupport.}
If a distribution only has a single successor, we can trivially conclude that the transition probability equals 1.
If a distribution has two successors, it suffices to estimate only one of the two probabilities:
Upon estimating $p_1 \in \hat{p} \pm c$, we obtain $p_2 \in (1 - \hat{p}) \pm c$.
Notably,
the confidence in the estimation of $p_1$ transfers to $p_2$: The estimation of $p_2$ is correct exactly when $p_1$ is correct.
Thus, we only need to learn one transition probability instead of two, saving budget.
While this observation seems rather trivial, even this is not exploited by state-of-the-art SMC approaches \cite{AKW19,WeiningerGMK21,agarwal2022pac,DBLP:journals/jair/BadingsRAPPSJ23,DBLP:journals/corr/abs-2310-12248}.

\begin{remark}
	This reasoning does not easily extend to larger distributions:
	Given $k$ successors, one could estimate $k - 1$ probabilities and infer the $k$-th from the rest.
	However, the resulting error bound for the $k$-th estimate is worse:
	In a three-successor distribution where $p_1 \in \hat{p_1} \pm c$ and $p_2 \in \hat{p_2} \pm c$ we only obtain $p_3 \in (1 - \hat{p_1} - \hat{p_2}) \pm \textbf{2} c$.
	This also relates to \cref{sec:3-die}, where we argue that focusing on individual probabilities rather than entire distributions seems to be beneficial.
\end{remark}

\newcommand{\eventStateActionCorrect}[2]{\text{Corr}(s, a)}
\paragraph{\independence.}
By their nature, the transition distributions in Markov systems are independent.
Thus, the events of correctly estimating the probabilities for different state-action pairs are independent.
Consequently, instead of dividing the confidence budget additively, we can divide it multiplicatively over the distributions (only utilizing the union bound for the transitions comprising each single state-action pair).
We prove in \cref{app:independence}:
\begin{restatable}{proposition}{lemmaindependence}\label{prop:independence}
	Let $\mathcal{D}$ be the set of all distributions to be learnt, $\delta$ the confidence budget and $\delta_d$ the confidence budget of a single distribution $d$.
	If $\prod_{d\in\mathcal{D}} (1-\delta_d) \geq \delta$, then the probability of correctly estimating all distributions is larger than $\delta$.
\end{restatable}

\subsection{Using Information About the Property}\label{sec:4-use-prop}
\noindent
Now, we discuss optimizations specific to exploiting the structure of the given systems \emph{relative to the given property}.

\begin{figure}[t]
	\centering
	\begin{tikzpicture}[xscale=1.1,yscale=0.9]
		\node[state] at (0.25,0) (init) {$\initialstate$};
		\node[actionnode] at (0.75,0) (inita) {};
		\node[state] at (1.75,0.5) (s1) {$s_1$};
		\node[actionnode] at (2.4,0.6) (s1a) {};
		\node[state] at (3.25,0.5) (s2) {$s_2$};
		\node[actionnode] at (3.5,-0) (s2a) {};

		\node[state] at (4.5,0.5) (sink) {$\textbf{0}$};
		\node[state] at (4.5,-0.5) (goal) {$\textbf{1}$};

		\node[state] at (1.75,-0.5) (t) {$t$};
		\node[actionnode] at (2.5,-0.6) (ta) {};
		\node[actionnode] at (2.25,-0.1) (tb) {};

		\begin{scope}[on background layer]
			\node[fit=(s1) (s2),rounded corners,rectangle,fill=gray!30,inner sep=5pt] {};
		\end{scope}

		\path[-]
		(init) edge (inita)
		(s1) edge (s1a)
		(s2) edge (s2a)
		(t) edge (ta)
		(t) edge (tb)
		;

		\path[->]
		(inita) edge (s1)
		(inita) edge (t)

		(s1a) edge[out=40,in=40,looseness=1.5] (s1)
		(s1a) edge[out=10,in=170] (s2)

		(s2a) edge[out=-55,in=170] (goal)
		(s2a) edge[out=-40,in=-130] (sink)
		(s2a) edge[out=-70,in=0,looseness=0.75] (t)

		(s2) edge[bend left=10] (s1)

		(ta) edge[out=-10,in=-25,looseness=2] (t)
		(ta) edge[out=-5,in=190] (goal)

		(tb) edge (s1)
		(tb) edge (s2)

		(goal) edge[loop right] (goal)
		(sink) edge[loop right] (sink)
		;
	\end{tikzpicture}
	\caption{
		A small MDP to illustrate several potential savings that can be obtained through \equivalencestructures.
		The boxed states form an end component, \textbf{1} denotes the designated target state.
		We omit transition probabilities, as these are also not visible to our algorithm.
		We also omit action labels for readability.
	} \label{fig:equivalence}
\end{figure}

\paragraph{\equivalencestructures.}
Classical verification uses several graph structural analyses which so far have been neglected for SMC.
A very general analysis is the \emph{never-worse relation}~\cite[Def.~5]{LeRPer18}.
Intuitively, a state $s_1$ is never worse than $s_2$ if we have $\val_{\MDP}(s_1)\geq\val_{\MDP}(s_2)$ independently of the exact transition probabilities.
If additionally $s_2$ is also never worse than $s_1$, these states form an equivalence class and can be merged~\cite[Thm.~1]{LeRPer18}, reducing the number of transitions.
Since deciding this relation is coNP-complete even on Markov chains~\cite[Thm.~4]{LeRPer18}, we utilize special cases which we can identify efficiently in our setting~\cite[Sec.~3.2]{LeRPer18}: attractors, maximal end components, and dominated states.

\begin{example}\label{ex:equivalence-structures}
	We illustrate \equivalencestructures and its synergy with other optimization using the example in \cref{fig:equivalence}.
	Firstly, state $t$ is never worse than the goal state $\textbf{1}$:
	using the lower action, $\textbf{1}$ is reached almost surely, independent of the exact probabilities.
	Consequently, we do not need to
	estimate \emph{any} action of $t$.
	Computing such value-1 (and value-0) states is a well-known graph analysis, see e.g.~\cite[Chp.~10.6.1]{DBLP:books/daglib/0020348}.
	Secondly, as $t$ and \textbf{1} have the same value, for the lower action of $s_2$ we only need to estimate the probability of moving to $\{t,\textbf{1}\}$, not each individually.
	Combined with \smallsupport, we can from this infer the probability of moving to \textbf{0}.
	Thirdly, $s_1$ and $s_2$ can mutually reach each other with probability 1 and hence achieve the same value (they form an end component~\cite[Chp.~3.3]{de1998formal}).
	Thus, we do not need to learn their \enquote{internal} transitions (in particular, the upper action of $s_1$).
	Overall, instead of dividing the confidence budget over all twelve transitions (as the state-of-the-art does), we can focus on only two transitions (one of $\initialstate$ and the one from $s_2$ to $\{t,\textbf{1}\}$) instead of eleven.
\end{example}

\newcommand{\entries}{\text{In}}
\newcommand{\exits}{\text{Out}}
\newcommand{\choices}{\text{Choice}}

\begin{figure}[t]
	\centering
	\begin{tikzpicture}
		\node[state] at (0.6,1) (init) {$\initialstate$};
		\node[actionnode] at (1.1,1) (inita) {};
		\node[state] at (2,1) (s1) {$s_1$};
		\node[actionnode] at (2.4,0.6) (s1a) {};
		\node[state] at (2.5,0) (s2) {$s_2$};
		\node[actionnode] at (2.8,0.7) (s2a) {};
		\node[state] at (5.75,0.5) (s3) {$s_3$};
		\node[actionnode] at (5.75,1.3) (s3a) {};
		\node[actionnode] at (5.75,-0.1) (s3b) {};
		\node[state] at (4.5,0) (s4) {$s_4$};
		\node[actionnode] at (3.7,-0.2) (s4a) {};
		\node[state] at (7,1) (e1) {$e_1$};
		\node[state] at (7,0) (e2) {$e_2$};

		\begin{scope}[on background layer]
			\node[fit=(s1) (s2) (s3) (s4),rounded corners,rectangle,fill=gray!30,inner sep=5pt] {};
		\end{scope}

		\path[-]
		(init) edge (inita)
		(s1) edge (s1a)
		(s2) edge (s2a)
		(s3) edge node[anchor=west,action] {$a$} (s3a)
		(s3) edge node[anchor=west,action] {$b$} (s3b)
		(s4) edge (s4a)
		;

		\path[->]
		(inita) edge (s1)
		(s1a) edge[out=-40,in=80] (s2)
		(s1a) edge[out=-40,in=160] (s4)
		(s1a) edge[out=-40,in=-180,looseness=0.5] (s3)

		(s2a) edge[out=80,in=30,looseness=2] (s2)
		(s2a) edge[out=80,in=130] (s4)
		(s2a) edge[out=80,in=0] (s1)

		(s3a) edge[out=180,in=10] (s1)
		(s3a) edge[out=-20,in=180] (e1)
		(s3a) edge[out=-20,in=130] (e2)

		(s3b) edge[out=180,in=-5] (s4)
		(s3b) edge[out=180,in=-130] (s3)
		(s3b) edge[out=0,in=230] (e1)
		(s3b) edge[out=0,in=190] (e2)

		(s4a) edge[out=-170,in=-20] (s2)
		(s4a) edge[out=-180,in=-180,looseness=2] (s4)
		;
	\end{tikzpicture}\hspace{1cm}
	\begin{tikzpicture}
		\node[state] at (-0.25,0) (s1) {$\hat{s}$};
		\node[cloud,inner sep=0.1cm,aspect=2,cloud puffs=20,fill=gray!10,draw] at (1.5,0.5) (clouda) {};
		\node[cloud,inner sep=0.1cm,aspect=2,cloud puffs=20,fill=gray!10,draw] at (1.5,-0.5) (cloudb) {};
		\node[state] at (2.75,0.5) (e1) {$e_1$};
		\node[state] at (2.75,-0.5) (e2) {$e_2$};

		\node at (0,-0.9) {};

		\path[-]
		(s1) edge[sloped] node[anchor=south,action] {$s_3 \mapsto a$} (clouda)
		(s1) edge[sloped] node[anchor=north,action] {$s_3 \mapsto b$} (cloudb)
		;
		\path[->]
		(clouda) edge[out=5,in=175] (e1)
		(clouda) edge[out=0,in=130] (e2)

		(cloudb) edge[out=0,in=230] (e1)
		(cloudb) edge[out=0,in=190] (e2)
		;
	\end{tikzpicture}
	\caption{
		Left: An MDP with a fragment $R$ highlighted. Right: Fragment-quotient.
	} \label{fig:scc_fragment}
\end{figure}

\paragraph{\fragments.}
We introduce a new optimization which identifies sets of states where the \enquote{internal} behaviour is not (too) interesting, and thus can be abstracted.
\begin{example}\label{ex:4-fragments}
	Consider the MDP in \Cref{fig:scc_fragment} (left).
	The marked area of the state space has a lot of internal structure; however, we only care about how we could \emph{leave} this area.
	In this example, there are only two possibilities, namely through the actions $a$ or $b$ in state $s_3$.
	Intuitively, we only are interested in the \enquote{big step} behaviour of the system as depicted in \Cref{fig:scc_fragment} (right).
	So, once we identify this fragment, we only need to estimate two probabilities, namely the probability to reach $e_1$ under $a$ and $b$, respectively, instead of $>10$ distinct values.
\end{example}
Generally, for a set of states $R \subseteq \States$ (a fragment), we replace all internal transitions as follows:
Let $\StrategiesMD<\MDP>(R)$ be all MD-strategies restricted to $R$, i.e.\ all possible choices of actions for states in $R$.
We replace every transition entering $R$, by a \enquote{macro action} for every strategy $\pi \in \StrategiesMD<\MDP>(R)$ that immediately leads to the states outside the fragment.
(Practically, this means sampling a path under the internal behaviour $\pi$ until one of the exits is reached.)
Thus, we omit all internal transitions, instead having a transition for the \enquote{big-step} behaviour.
Formally:
\begin{definition}\label{def:4-fragments}
	Fix an MDP $\MDP$, a goal set $\Reachset$, a set of states $R \subseteq \States \setminus T$.
	The fragment-quotient $\MDP_R$ is obtained by defining $S_R=S\setminus R$ and for all $s\in S_R$ and $r\in R$ replacing all $\mdptransitions(s, a, r)> 0$ by $\mdptransitions_R(s, (a,\pi), s') \coloneqq \ProbabilityMDP<\MDP, s><\pi>[\{\finitepath \mid {\finitepath_0=s} \land {\exists i>1}. ~\finitepath_i = s' \land \forall j \in [0,i-1].~\finitepath_j \in R\}]$ for all $s' \in S_R$.
\end{definition}
For correctness, we require that for any strategy the probability to eventually leave $R$ is 1, i.e.\ it is not an end component.
(This case can be detected and treated by using \equivalencestructures.)
We prove in \Cref{app:fragments}:

\begin{restatable}{proposition}{lemmafragments} \label{stm:fragments}
	Fix an MDP $\MDP$ with goal set $\Reachset$ and set of states $R \subseteq \States \setminus T$ which does not contain an end component, and let $\MDP_R$ be the fragment-quotient.
	For all states $s\in \States$, we have $\val_{\MDP}(s) = \val_{\MDP_R}(s)$.
\end{restatable}
\noindent
It remains to decide which set $R$ is a good candidate for a fragment.
The number of transitions for $R$ in the fragment quotient is the number of internal strategies multiplied with the number of entries and exits.
Naively, we can solve a global optimization problem to find the fragments which minimize the overall number of transitions.
However, since we want the improvements to be easy to implement and fast, we instead utilize two other candidates for fragments:
Chains (states with a single predecessor), and strongly connected components (SCC, a well-known graph theoretic notion), both of which can be determined in linear time.
In \Cref{app:fragments}, we provide more details, and in \Cref{app:other-objectives} we outline how to extend the construction to objectives other than reachability.

\section{Evaluation}\label{sec:5-title}

\noindent
In this section, we experimentally evaluate the effectiveness of our methods.
Our implementation is available at \url{https://doi.org/10.5281/zenodo.15231337}.

\subsection{Setup}
\noindent
We implemented our statistical methods as a prototype in Python and use PET \cite{DBLP:conf/cav/MeggendorferW24} for graph analysis and preprocessing.
As models, we consider reachability instances from the PRISM benchmark suite \cite{prism-benchmarksuite}, removing \enquote{trivial} ones, i.e.\ where the result is equal to 0 or 1.

We fix confidence $\delta=0.1$ and choose the desired precision $\varepsilon$ depending on the model such that the baseline also terminated reasonably quickly, in order to keep the evaluation time manageable (see \Cref{app:prop-smc-setup} for further details).
We solve the models using SMC as described in \cref{sec:2-smc}.
We employ a uniform sampling strategy, i.e.\ starting from the initial state we choose a random action, sample a successor, and repeat, until we reach a goal or sink state.
Then we estimate transition probabilities and solve the resulting interval MDP with interval iteration~\cite{HM18,GRIP}.
For the second step, we use different approaches: (i)~the baseline, using the state-of-the-art model-based SMC technique of splitting the confidence budget uniformly over all transitions and applying Hoeffding's inequality to every transition probability, as in e.g.~\cite{AKW19,agarwal2022pac,DBLP:journals/jair/BadingsRAPPSJ23}; (ii) our approach, utilizing all improvements from \Cref{sec:3-title,sec:4-title} (\cref{tbl:results}); and (iii) our approach, but leaving out single improvements for an ablation study (\cref{tbl:ablation}).
For \fragments we only use \chainfragments since we did not find any SCCs that would reduce the number of transitions.
We run every approach 10 times for each model and record the average number of sampled paths for each approach (see \cref{app:exp-details} for full results).
Note that we specifically quantify the sample complexity in terms of paths rather than transitions since the improvements of \Cref{sec:4-use-prop} modify the graph structure and hence the number of samples per transition is not directly comparable.

\subsection{Results and Discussion}
\noindent

\begin{table}[t]
	\centering
	\caption{
		Summary of our improvements' effects on the complete SMC-algorithm. For each instance, we report the number of transitions in the model; the number of sampled runs needed for the precision $\varepsilon$ with the baseline approach and with our improvements; and the improvement factor, i.e.\ ratio of the previous two values.}\label{tbl:results}
	\setlength{\tabcolsep}{5pt}
	\begin{tabular}{>{\ttfamily}l>{\ttfamily}lrrrrr}
		model & {objective} &  $\varepsilon$ & transitions & {baseline} & ours & ratio \\ \midrule
		consensus & disagree & 0.3 & 484 & 38\,090 & 14\,766 & 2.6 \\
		csma & all{\textunderscore}before & 0.1 & 1276 & 64\,310 & 783 & 82.1 \\
		firewire{\textunderscore}dl & deadline & 0.05 & 17\,417 & 1\,685\,965 & 46\,286 & 36.4 \\
		pacman & crash & 0.1 & 84 & 5\,048 & 181 & 28.0 \\
		wlan & collisions{\textunderscore}max & 0.7 & 5444 & 354\,949 & 9\,118 & 38.9 \\
		wlan{\textunderscore}dl & deadline & 0.9 & 326\,883 & 21\,083\,954 & 201\,536 & 104.6 \\
		zeroconf & correct & 0.05 & 953 & 18\,202 & 164 & 111.1 \\
		zeroconf{\textunderscore}dl & deadline & 0.05 & 4777 & 48\,248 & 64 & 749.2 \\
	\end{tabular}
\end{table}

\begin{table}[t]
	\centering
	\caption{
		Ablation study for our improvements with minimum, average (geometric mean), and maximum sample reduction factor when removing each improvement.
		The final line shows overall improvements over the baseline approach.
	}\label{tbl:ablation}
	\setlength{\tabcolsep}{5pt}
	\begin{tabular}{lrrrclrrr}
		Ablated         &  min &   avg &    max & \hspace{0.14cm} & Ablated                                &  min &   avg &    max \\ \midrule
		Clopper-Pearson & 1.09 &  2.90 &  14.71 &                 & Structural impr.                       & 1.06 & 13.79 & 252.42 \\
		\smallsupport   & 1.17 &  1.48 &   1.79 &                 & $\hookrightarrow$ \texttt{Equiv. Str.} & 1.00 &  8.20 & 143.39 \\
		\independence   & 1.01 &  1.23 &   1.82 &                 & $\hookrightarrow$ \texttt{Chain Frag.} &    1 &  1.71 &   7.89 \\ \midrule
		Baseline        & 2.58 & 54.07 & 749.19 &                 &                                        &      &       &
	\end{tabular}
\end{table}

\noindent
\paragraph{Runtime.}
Our analysis focusses on sample count rather than runtime, as in most realistic scenarios the cost of gathering samples dominates other computations.
Still, we mention that even in our prototype the overhead introduced by our improvements (e.g.\ removing chains and collapsing MECs) is negligible: for our largest model, it took roughly 12 seconds.
As we report in \cref{app:exp-details}, for all models in our evaluation, our improvements reduced the runtime of the SMC algorithm, with the speed-up ranging from approximately 1.2 up to 80.
This has several reasons:
Primarily, the structural improvements of \cref{sec:4-use-prop} reduce the size of the model, thereby speeding up solving the inferred models.
Moreover, the baseline requires more samples, and thus also more time for gathering these.
These savings massively outweigh the computational overhead of structural improvements.

\paragraph{All improvements.}
\cref{tbl:results} shows that our improvements always reduce the number of required samples.
The number of samples as compared to the baseline approach is less than 50\% on all models, less than 10\% on all but one, and even less than 1\% on three (two orders of magnitude improvement).

\paragraph{Ablation study.}
Investigating individual effects of improvements, we report the ratio of paths when leaving a single improvement out over all improvements for all models in \cref{app:exp-details}.
\cref{tbl:ablation} aggregates over models, reporting minimum, geometric average and maximum of the improvement factors.
As predicted in \cref{sec:3-sample-number}, Clopper-Pearson outperforms Hoeffding, with the effect being most pronounced for models where small transition probabilities are present, in line with \cref{fig:sample-complexity} (right).
\independence{} and \smallsupport consistently reduce sample count with no computation overhead.
The structural improvements (\equivalencestructures{} and \fragments{}) of course depend on the structure of the particular system at hand.
They are significant in most cases and they positively impact each other with \equivalencestructures giving rise to additional chains.
Additionally, all our suggested improvements lead to a smaller (or equivalent) sample complexity, i.e. \emph{never make it worse}.

\section{Related Work and Extensions}\label{sec:RW}

Related results of statistics are discussed in-depth in \cref{sec:3-title}.
The recent \cite{WatS} builds on a preprint of our work to categorize existing approaches based on their underlying statistical methods.
However, their focus lies on purely stochastic Markov chains and they mainly use statistical methods to tackle extremely large models, while we explicitly focus on non-determinism.
For a history of statistical approaches dealing with stochastic systems, we refer to the surveys~\cite{DBLP:conf/isola/Kretinsky16,DBLP:journals/tomacs/AghaP18,DBLP:series/lncs/LegayLTYSG19}.

We detail one direction of related work, namely algorithms with weaker guarantees than PAC.
Approaches for MDP such as \cite{DBLP:conf/qest/HenriquesMZPC12,DBLP:journals/sttt/LassaigneP15,DBLP:conf/tacas/LukinaEHBYTSG17,DBLP:conf/tacas/HahnPSSTW19,DBLP:conf/tacas/HahnPSSTW23} heuristically find promising strategies but in general cannot guarantee the chosen strategy is close to the optimum, thus only providing a lower bound on the value.
The model-free approach of lightweight scheduler sampling~\cite{lss-first,lss-more,DBLP:conf/isola/LegayST16,lss-more2} only guarantees that the estimate is close to the optimum if good strategies are frequent and also only provides lower bounds.
The reinforcement learning algorithms in~\cite{DBLP:conf/nips/GhavamzadehPC16,WieSuiSim23} guarantee an improvement over the policy that is used for obtaining the samples.
These and our paper work in a setting where we are allowed to obtain samples from any part the unknown system, even unsafe states.
In contrast, in a \enquote{safe online} setting, a good policy has to be computed \emph{while executing the system}, and unsafe states have to be avoided.
There, we refer to works on PAC online learning~\cite{concur18,uai20}, shielding~\cite{DBLP:conf/aaai/AlshiekhBEKNT18}, and regret minimization~\cite{DBLP:conf/colt/WagenmakerSJ22}.

Finally, we discuss how our improvements are (nearly) universally applicable, including methods where SMC is invoked as an inner step such as parameter synthesis \cite{DBLP:conf/cmsb/MolyneuxA20,10.1371/journal.pone.0291151}, and \emph{largely independent of the setting}.
Concretely, there exist several variants of our problem, depending on the assumptions about the unknown MDP, where our improvements also are applicable:

\textit{Topology knowledge.}
Either we know at least the topology of the underlying graph of the MDP (called {grey-box}) or even that has to be inferred from the simulations (called {black-box})~\cite[Def.~2]{AKW19}.
We focus on the grey-box, but detail in \Cref{sec:black-box} how our methods are also applicable in the black-box setting.

\textit{Sampling access.}
Assumptions about the sampling access to the MDP include, from least to most restrictive: sampling any state-action pair~\cite{DBLP:conf/cav/YounesS02,DBLP:conf/rss/FuT14}, running simulations through the system that can only be restarted in the initial state~\cite{atva14,AKW19,agarwal2022pac,DBLP:conf/nips/SuilenS0022}, or batch learning, where we only get a fixed data set of past interactions~\cite{DBLP:books/sp/12/LangeGR12,DBLP:conf/nips/GhavamzadehPC16,DBLP:conf/icml/Shi0W0C22,WieSuiSim23}.
Our methods are agnostic of the sampling method and just \enquote{make the most of the data} they get.

\textit{Objectives.}
MDPs can be complemented with various objectives, e.g., reachability, mean payoff, (discounted) total reward, or linear temporal logic.
Most of our improvements are independent of the objective, and we sketch in \Cref{app:other-objectives} how to generalize the few that do utilise information about the objective.

\section{Conclusion}
\noindent
We presented several improvements for the foundations of statistical model checking.
Overall, we suggest to use the Clopper-Pearson interval for estimating single transitions, and to exploit knowledge about the structure of the MDP and the property of interest where possible.
Our presented methods all are fast, can only improve the precision of the resulting intervals, and often do so significantly.
Thus, every implementation of a model-based SMC algorithm should use them.

In settings where samples are very expensive, we can employ more time-consuming improvements, for example computing the full never-worse relation~\cite{LeRPer18} or employing a global search to find the optimal fragments (see \cref{sec:4-use-prop}).
As future work, we aim to develop heuristics for focussing the confidence budget on \enquote{important} states.
Concretely, we plan to use stochastic gradient descent to find those states where giving them more of the confidence budget increases the precision for the overall value the most.

\newpage

\newpage

\bibliographystyle{splncs04}
\bibliography{main}

\newpage
\appendix

\crefalias{section}{appendix}
\crefalias{subsection}{appendix}

\section{Other Confidence Intervals Used in Verification Literature}\label{app:other-conf-intervals}

\noindent In this section, we comment on all the instances we are aware of where other concentration inequalities than Hoeffding's were applied in model-based SMC, namely \cite{BazilleGJS20,DBLP:conf/nips/SuilenS0022,DBLP:journals/jair/BadingsRAPPSJ23,DBLP:journals/tse/BuS24}.
Moreover, in \cref{subsec:wilson} we detail the Wilson score with continuity correction that is often claimed to provide minimum coverage probabilities that closely align with the nominal confidence level $1-\delta$ \cite{bcd-comment-casella}, and we show why this is not sufficient for model-based SMC.
Before that, we remark that unsound methods are also applied in SMC that is not model-based: For example, in~\cite{Arnd}, the Agresti-Coull interval~\cite{agresti1998approximate} is applied. It only provides a guarantee on the average coverage probability and thus compromises the statistical guarantees of the algorithms developed in~\cite{Arnd}.

\begin{itemize}
	\item In~\cite{BazilleGJS20}, the authors employ the so-called Chen bound.
	In \cref{app:chen}, we discuss in more detail why it is not proven to provide the necessary statistical guarantees.
	\item In~\cite{DBLP:conf/nips/SuilenS0022}, the authors use linearly updating probability intervals.
	These do not provide PAC-guarantees, as their main goal is to efficiently handle changing environment dynamics.
	Thus, they are not applicable in our setting where we require a PAC-guarantee.
	\item In~\cite{DBLP:journals/jair/BadingsRAPPSJ23}, the authors developed a new scenario-based approach that it is in fact a strictly weaker version of the Clopper-Pearson Interval.
	We provide a formal proof of this statement below in \cref{app:jansen}.
	\item In~\cite{DBLP:journals/tse/BuS24}, the authors provide a sequential variant of the Clopper-Pearson interval, offering a way to improve the efficiency of SMC-algorithms.
	Importantly, their proof of soundness~\cite[Thm.\ 1]{DBLP:journals/tse/BuS24} is conditional on the hypothesis that the width of the interval is maximized when the number of observed successes is half of the number of samples~\cite[Hypoth.\ 1]{DBLP:journals/tse/BuS24}.
	Our \Cref{lemma:sample-complexity} shows their hypothesis to be true, thereby completing their soundness proof.
\end{itemize}
Throughout the rest of this section, we refer to a binomial distribution with success probability $p$ and a test sequence on it with $n$ trials and $k$ successes.
We use $\hat{p}=\frac{k}{n}$ as a shorthand notation for the maximum likelihood estimate of $p$.

\subsection{Wilson Score Interval with Continuity Correction} \label{subsec:wilson}
\noindent
We put special emphasis on the Wilson score with continuity correction (CC) as it is often cited in the statistics literature as providing a minimal coverage probability that aligns with the nominal $1-\delta$ confidence level \cite{bcd-comment-casella}.
While it was already stated by Newcombe that soundness is not strictly guaranteed when he introduced this methods \cite{New98}, we discuss why this is especially problematic in the context of model-based SMC.

\paragraph{Definition.}
One of the most famous and widely used (see~\cite{agresti1998approximate,BroCaiDas01,DBLP:journals/jql/Wallis13}) methods is the Wilson score interval~\cite{Wil27}.
However, the Wilson score only guarantees that the \emph{average} coverage probability is above the desired confidence level $1-\delta$ \cite{BroCaiDas01}.
Later works address this shortcoming and introduce a \emph{continuity correction} (CC) for the Wilson score interval \cite{New98}.
It is computed as follows for $k$ successes from $n$ samples:
\begin{align*}
	\underline{p}_{wcc}& =
	\big(2n{\hat {p}}+z^{2}-z{\sqrt {z^{2}-{\tfrac{1}{n}}+4n{\hat {p}}(1-{\hat {p}})+(4{\hat {p}}-2)}}-1\big) / (2(n+z^{2})) \\
	\overline{p}_{wcc}& =
	\big(2n{\hat {p}}+z^{2}-z{\sqrt {z^{2}-{\tfrac{1}{n}}+4n{\hat {p}}(1-{\hat {p}})-(4{\hat {p}}-2)}}+1) / (2(n+z^{2}))
\end{align*}
where $z$ is the $1-\frac{\delta}{2}$ quantile of the standard normal distribution.
When $k=0$ (resp.\ $k=n$) the lower (resp.\ upper) bound is taken as $0$ (resp.\ $1$) instead.
Alternatively, this can equally be seen as a Wilson score interval (without continuity correction) where $k$ is replaced by $k+1/2$ for the lower bound, and by $k-1/2$ for the upper bound \cite{Wallis20}.

\paragraph{Unsoundness.}
\begin{figure}[t]
	\centering
	\begin{minipage}{0.45\textwidth}
		\centering
		\resizebox{\columnwidth}{!}{
		\begin{tikzpicture}
			\begin{axis}[xlabel=$p$,ylabel=coverage probability ({$\delta=0.1$})]
				\addplot [no marks,draw=black] table [col sep=comma,x index=0, y index=1] {data/coverage_fixed_n_wilson_cc_delta_09.csv};
				\addplot [no marks,draw=red,thick] coordinates {(0,0.9) (1,0.9)};
			\end{axis}
		\end{tikzpicture}
		}
	\end{minipage}\,
	\begin{minipage}{0.45\textwidth}
		\centering
		\resizebox{\columnwidth}{!}{
		\begin{tikzpicture}
			\begin{axis}[xlabel=$p$,ylabel=coverage probability ({$\delta=0.01$})]
				\addplot [no marks,draw=black] table [col sep=comma,x index=0, y index=1] {data/coverage_fixed_n_wilson_cc_delta_099.csv};
				\addplot [no marks,draw=red,thick] coordinates {(0,0.99) (1,0.99)};
			\end{axis}
		\end{tikzpicture}
		}
	\end{minipage}
	\caption{Coverage probabilities of Wilson score interval with CC for $n=100$ and $\delta=0.1$ (left) and $\delta=0.01$ (right). The red line marks the desired coverage of $1-\delta$. For each $p$, the y-value gives the probability that after $n=100$ samples the interval computed the Wilson score with CC contains $p$. For $\delta=0.1$ this is consistently above $1-\delta$ whereas for $\delta=0.01$ it is not for $p$ close to $0$ or $1$. We used the artifact of \cite{WatS} to produce these plots.\label{fig:wilsoncc_coverage}}
\end{figure}
Newcombe's paper in which he introduced the Wilson score with CC \cite{New98} already analyses the coverage probability for various methods, including the Wilson score with CC, for different parameters.
Crucially, it only gives coverage probabilities for varying true probabilities $p$ and sample sizes $n$ are considered while keeping the confidence level $1-\delta$ fixed at $0.95$.
However, as we show in \Cref{fig:wilsoncc_coverage}, the deviation of the minimal coverage probability of the Wilson score with CC from the desired coverage $1-\delta$ massively depends on $\delta$ itself.
As such, while the Wilson score with CC is actually sound for $n=100$ and $\delta=0.1$, it massively undershoots the desired coverage for $\delta=0.01$, only achieving a coverage probability of $<0.97$ rather than the desired $0.99$ for $p$ close to $0$ or $1$.
We are not aware of any conditions under which it is guaranteed that the Wilson score with CC achieves the desired coverage for all $p$.
Generally, we suspect that the issue of the coverage probability misaligning for small $\delta$ is only emphasized (relative to $\delta$) for even smaller $\delta$.
In the context of SMC however, we ultimately have to distribute the confidence budget $\delta$ over many transitions, resulting in a desired coverage probability very close to 1.
In contrast, if for small $\delta$ the Wilson score with CC tends to yield less reliable coverage probabilities, this means the probability of wrongly estimating at least one transition in the MDP is larger than $\delta$, meaning we can no longer give PAC guarantees on the result.

\subsection{Details on the Chen Bound Used in~\cite{BazilleGJS20}}\label{app:chen}

\noindent
The bound in~\cite[Thm.\ 1]{BazilleGJS20} builds on the algorithm of Chen~\cite{chen}.
However,
the analysis of Chen's algorithm assumes that the true probability $p$ satisfies $p\leq \frac{1}{2} - \varepsilon$ or $p \geq \frac 1 2 + \varepsilon$ for the given $\varepsilon>0$.
Additionally, formal PAC guarantees are only proven for the upper bound when $p\leq \frac{1}{2} - \varepsilon$, and for the lower bound when $p \geq \frac 1 2 + \varepsilon$, with the respective other bound only being shown empirically.
This is the case in both \cite[Thms.~1 and 2]{chen}, as well as Chen and Xu's related work on combining adaptive sampling and the Massart inequality~\cite[Thm.\ 1]{chen13} which resembles~\cite[Thm.~1]{BazilleGJS20} more closely than the theorems in the work they cite~\cite{chen}.
	In particular, based on Chen's analysis, there  is no formal guarantee for what happens if $p=\frac 1 2$.
Thus, we refrain from considering the bound proposed in~\cite[Thm.\ 1]{BazilleGJS20} because we found no proof that it gives the necessary formal guarantees.

\subsection{Method of~\cite{DBLP:journals/jair/BadingsRAPPSJ23} Strictly Weaker than Clopper-Pearson Interval}\label{app:jansen}

\noindent
Here, we prove our claim that the confidence method introduced and proven in~\cite{DBLP:journals/jair/BadingsRAPPSJ23} is a weaker version of the Clopper-Pearson interval.
Specifically, we show that the Clopper-Pearson interval is always a sub-interval of their computed interval.

\begin{proof}
	\noindent
	Let us start by restating Theorem 1 in~\cite{DBLP:journals/jair/BadingsRAPPSJ23}.
	We rephrase it slightly for consistency with the notation throughout our paper.
	In addition to changing variable names, Badings et~al.\ formulate the problem in terms of the number of failures (which they call $N^{out}_{s'}$) whereas we formulate the problem in terms of successes ($N-N^{out}_{s'}$) here.

	\begin{theorem}[Theorem 1 in \cite{DBLP:journals/jair/BadingsRAPPSJ23}]
		For $n\in \mathbb{N}$ samples from a binomial distribution, fix a confidence parameter $\delta\in (0,1)$. Given $k$ successes, the success probability $p$ is bounded by
		\begin{equation*}
			\Probability[\underline{p} \leq p \leq \overline{p} \mid \hat{p}=\tfrac{k}{n}] \geq 1 - \delta
		\end{equation*}
		where $\underline{p}=0$ if $k=0$, and otherwise $\underline{p}$ is the solution of
		\begin{equation*}
			\frac{\delta}{2n} = \sum_{i=0}^{n-k} \binom{n}{i}(1-\underline{p})^i\underline{p}^{n-i}
		\end{equation*}
		and  $\overline{p}=1$ if $k=n$, and otherwise $\underline{p}$ is the solution of
		\begin{equation*}
			\frac{\delta}{2n} = 1-\sum_{i=0}^{n-k-1} \binom{n}{i}(1-\underline{p})^i\underline{p}^{n-i}
		\end{equation*}
	\end{theorem}

	\noindent
	We focus on the lower bound first.
	Recall that for the Clopper-Pearson interval we have
	\begin{equation*}
		\underline{p}_{cp} = I^{-1}_{\delta/2}(k,n-k+1)
	\end{equation*}
	Therefore, by definition of the regularised beta function $I_z$ and its inverse $I^{-1}_z$~\cite{Tem92} we have
	\begin{align*}
		\frac{\delta}{2} & = I_{\underline{p}_{cp}}(k,n-k+1) \\
		& = 1 - \Probability[X \leq k-1  \mid X \sim \mathit{Bin(n,\underline{p}_{cp})}] \\
		& = 1 - \sum_{i=0}^{k-1} \binom{n}{i}(1-\underline{p}_{cp})^{n-i}\underline{p}_{cp}^i \\
		& = \sum_{i=0}^{n-k} \binom{n}{i}(1-\underline{p}_{cp})^i\underline{p}_{cp}^{n-i}
	\end{align*}
	Similarly, for the upper bound we have
	\begin{align*}
		1-\frac{\delta}{2} = & I_{\overline{p}_{cp}}(k+1,n-k) \\
		& = 1 - \Probability[X \leq k  \mid X \sim \mathit{Bin(n,\overline{p}_{cp})}] \\
		& = 1 - \sum_{i=0}^{k} \binom{n}{i}(1-\overline{p}_{cp})^{n-1}\overline{p}_{cp}^i \\
		& = \sum_{i=0}^{n-k-1} \binom{n}{i}(1-\overline{p}_{cp})^i\overline{p}_{cp}^{n-i} \\
		\iff \frac{\delta}{2} & =  1 - \sum_{i=0}^{n-k-1} \binom{n}{i}(1-\overline{p}_{cp})^i\overline{p}_{cp}^{n-i}
	\end{align*}
	Thus, the Clopper-Pearson interval directly proves that the solutions $\underline{p}$ and $\overline{p}$ to
	\begin{align*}
		\frac{\delta}{2} & = \sum_{i=0}^{n-k} \binom{n}{i}(1-\underline{p})^i\underline{p}^{n-i} \\
		\frac{\delta}{2}  & = 1 - \sum_{i=0}^{n-k-1} \binom{n}{i}(1-\overline{p})^i\overline{p}^{n-i}
	\end{align*}
	provide an interval $[\underline{p},\overline{p}]$ that contains $p$ with probability at least $1-\delta$.

	Therefore we can directly see that the method in~\cite{DBLP:journals/jair/BadingsRAPPSJ23} can equally be obtained by computing the Clopper-Pearson interval with confidence budget $\frac{\delta}{n}$, immediately making it a weaker version of the Clopper-Pearson Interval with equality only occurring when $n=1$.\qed
\end{proof}

\section{Proofs for \Cref{sec:3-title}}

\subsection{Equivalence of Weissman $L^1$-bound and Hoeffding Bound for Two-successor Transitions} \label{sec:app-l1-is-hoeffding}
\noindent
In \Cref{sec:3-die} we claim that the method for constructing the confidence region for the probabilities of a coin as an $L^1$-ball, as, e.g., in~\cite{StrLit04,AueJakOrt08,WieSuiSim23}, using the inequality provided by Weissman et al.~\cite{WeiOrdSer03} coincides with constructing a confidence interval for the probability vie Hoeffding's inequality \cite{Hoe63} as, e.g., in~\cite{AKW19,BaiDubWie23}, combined with our \emph{small support} improvement (\Cref{sec:4-use-model}).

Let us first restate the construction of the $L^1$-ball \cite{WieSuiSim23}:
Given a confidence budget $\delta$ and $n$ samples of a $k$-dimensional categorical distribution with probability vector $\mathbf{p}$, we have
\begin{equation*}
	\Probability[\lVert \mathbf{\hat{p}} - \mathbf{p} \rVert_1 \leq 1-\varepsilon_1] \geq 1-\delta
\end{equation*}
where
\begin{equation*}
	\varepsilon_1 = \sqrt{\frac{2(\log(2^{k}-2)-\log\delta)}{n}}
\end{equation*}
and $\mathbf{\hat{p}}$ is the maximum likelihood estimate of $\mathbf{p}$.
For a $k=2$ dimensional transition vector this simplifies to
\begin{equation*}
	\varepsilon_1 = \sqrt{\frac{-2\log\frac{\delta}{2}}{n}}
\end{equation*}
Next, explicitly rewriting $\mathbf{p}=\begin{bmatrix}p_1 & p_2\end{bmatrix}^{T}$ and $\mathbf{\hat{p}}=\begin{bmatrix}\hat{p}_1 & \hat{p}_2\end{bmatrix}^{T}$ and using $p_2=1-p_1$ and $\hat{p}_2=1-\hat{p}_1$ we have
\begin{equation*}
	\lVert \mathbf{\hat{p}} - \mathbf{p} \rVert_1 \leq 1-\varepsilon_1 = 2 \lvert p_1 - \hat{p}_1 \rvert
\end{equation*}
Finally, we can put this together as
\begin{equation*}
	\lVert \mathbf{\hat{p}} - \mathbf{p} \rVert_1 \leq \varepsilon_1 \quad \iff \quad \lvert p_1 - \hat{p}_1 \rvert \leq \sqrt{\frac{\log\frac{\delta}{2}}{-2n}}
\end{equation*}
which is exactly the interval for $p_1$ obtained through the two-sided Hoeffding inequality (see \Cref{sec:3-coin}).
As discussed in \Cref{sec:4-use-model}, we do not need to consider confidence interval for $p_2$ with Hoeffding's method separately as $p_2$ is fully dependent on $p_1$.

\subsection{Bennett's Inequality and Bernstein's Inequality}\label{sec:app-bennett}
Recall that both Bennett's inequality and Bernstein's inequality, similar to Hoeffding's inequality, can be applied to bound the sum of random variables.
Unlike Hoeffding's inequality, they also consider the variance of the random variables when constructing the confidence interval.
This is an improvement when bounding the sum of random variables with known distributions, but a hindrance when it comes to estimating $p$ since we now would also need to estimate the variance of the random variable.

First, we show that applying Bennett's equality with a naive variance estimation effectively is guaranteed to yield bigger confidence intervals than Hoeffding's inequality.

Bennett's inequality \cite{Ben62} states -- as a two-sided variant -- that for $n$ random variables $X_1,\dots,X_n$, bounded between $0$ and $a$, the deviation of their sum $S_n = \sum_{i=1}^n (X_i)$ from its expected value $\mathbb{E}(S_n)$ satisfies
\begin{equation*}
	\Probability[|S_n-\mathbb{E}(S_n)|\geq t] \leq 2\exp\left( -\frac{\sigma^2}{a^2}h\left(\frac{at}{\sigma^2}\right) \right)
\end{equation*}
for all $t\geq 0$ where $\sigma^2 = \sum_{i=1}^n (X_i-\mathbb{E}(X_n))^2$ is the sum of the variances of the random variables and $h(u)=(1+u)\log(1+u)-u$.

Applying this to estimate a binomial probability parameter $p$ that was sampled with $k$ successes out of $n$ trials yields $S_n = k-np$, and therefore $\frac{S_n}{n} = \frac{k}{n}-p = \hat{p}-p$
Further we know $a=1$.
As $p$ is unknown, we do not know $\sigma^2$ either. However, since for $n$ binomial random variables we have $\sigma^2=np(1-p)$, we get $\frac{n}{4}$ as a trivial upper bound on $\sigma^2$ we can apply.
Thus we can simplify Bennett's inequality to
\begin{equation*}
	\Probability\left[|\hat{p}-p|\geq \frac{t}{n}\right] \leq 2\exp\left( -\frac{n}{4}h\left(\frac{4t}{n}\right) \right) = 2\exp\left( -\left(\frac{n}{4}+t\right)\log\left(1+\frac{4t}{n}\right)+t \right)
\end{equation*}
Comparing this to the two-sided Hoeffding inequality which states \cite{Hoe63}
\begin{equation*}
	\Probability\left[|\hat{p}-p|\geq \frac{t}{n}\right] \leq 2\exp\left( -2t^2 \right)
\end{equation*}
we see that Bennett's inequality yields a higher confidence budget (and thus tighter confidence intervals when inverting the problem) when
$-\left(\frac{n}{4}+t\right)\log\left(1+\frac{4t}{n}\right)+t < -2t^2$ which does not have any solutions for $n\geq 1$.
Thus, we conclude that without prior knowledge on $p$ Hoeffding's inequality is always better.

Bernstein's (first) inequality is very similar to Bennett's inequality.
The only difference is that $h$ is replaced by $h'(u) =9\left(1+\frac{u}{3}-\sqrt{1+\frac{2u}{3}}\right) \leq h(u)$.
Other variants of defining $h'$ have also appeared in the literature, however all have in common that they are an under-approximation of $h(u)$.
The fact $h'$ is an under-approximation of $h$ directly implies that the confidence budget of Bernstein's inequality is worse than the one provided by Bennett's inequality.

\subsection{Proof of Proposition \ref{lemma:sample-complexity}}\label{app:proof-worst-case-sample}

\samplecomplexity*
\noindent

\begin{proof}
	We consider both methods separately.
	As previously, we denote the number of successes in $\mathcal{S}$ as $k$ and $\hat{p}=\frac{k}{n}$.

	\paragraph{Case I: Hoeffding's inequality.}
	From its definition (see \Cref{sec:3-coin}), it is clear that the interval computed from Hoeffding's inequality has size at most $2c_{ho}=2\sqrt{\frac{\ln 2/\delta}{2n}}$ but possibly less if either side of the bound is clipped to the extreme values of $0$ or $1$.
	To maximise the size of the confidence interval, we have to avoid the values of $p$ that are in the confidence interval $\hat{p}\pm c_{ho}$ but implausible, i.e. $p\not\in [0,1]$. This is achieved exactly when $\hat{p}=\frac{1}{2}$.

	\paragraph{Case II: Clopper-Pearson interval.}
	For the Clopper-Pearson interval we write the confidence width for a specific sample of size $n$ with $k$ successes, using the symmetry of the Clopper-Pearson interval, as
	\begin{align*}
		\varepsilon_{cp}^{k,n}&=I^{-1}_{1-\delta/2}(k+1,n-k)-I^{-1}_{\delta/2}(k,n-k+1) \\
		&= I^{-1}_{1-\delta/2}(k+1,n-k)-I^{-1}_{1-\delta/2}(n-k+1,k)-1
	\end{align*}
	where $I^{-1}_x(a,b)$ is the inverse regularized beta function defined as
	\[
		I^{-1}_x(a,b) = p \iff I_p(a,b)=x
	\]
	where $I_p(a,b)$ is the regularized beta function (see, e.g. \cite{Tem96}).
	Note that the regularized beta function is monotonously increasing with $p$ and always non-negative for $0\leq p \leq 1$.

	In order to show the claim, we will first show that $I_p(k,n-k)$ is in a discrete sense concave w.r.t. $k$ for large $p$ and convex w.r.t $k$ for small $p$.
	In particular, it will always be concave for the upper bound, and convex for the lower bound of the Clopper-Pearson interval.
	We then prove that for $I^{-1}_x(k,n-k)$ the opposite is the case, i.e. the upper bound behaves in a convex way and the lower bound in a concave way.
	Together with the symmetry of the Clopper-Pearson interval this yields that the maximal width is obtained at $n/2$ positive samples.

	To start, note that by definition the upper bound of the Clopper-Pearson interval $\overline{p}^{k,n}_{cp}$ satisfies
	\[
		I_{\overline{p}^{k,n}_{cp}}(k,n-k+1)=1-\delta/2.
	\]
	We now consider how much the left term deviates from $1-\delta/2$ if one sample more or one sample less was positive.
	Notice that we do not add or remove a sample here, but rather convert a negative one into a positive one (or vice verse), leaving the total number of samples unchanged.
	To expand the regularized beta function we use the recurrence relations $I_{x}(a+1,b)=I(a,b)-\frac{x^a(1-x)^b}{aB(a,b)}$ and $I_{x}(a,b+1)=I(a,b)+\frac{x^a(1-x)^b}{bB(a,b)}$ where $B$ is the beta function \cite[Eq. 11.37]{Tem96}:
	\begin{align*}
		I_{\overline{p}^{k,n}_{cp}}(k+1,n-k) &=I_{\overline{p}^{k,n}_{cp}}(k,n-k+1) - \frac{\left({\overline{p}^{k,n}_{cp}}\right)^k(1-{\overline{p}^{k,n}_{cp}})^{n-k}}{kB(k,n-k+1)} \\
		I_{\overline{p}^{k,n}_{cp}}(k-1,n-k+2)&=I_{\overline{p}^{k,n}_{cp}}(k,n-k+1) + \frac{\left({\overline{p}^{k,n}_{cp}}\right)^{k-1}(1-{\overline{p}^{k,n}_{cp}})^{n-k+1}}{(n-k+1)B(k,n-k+1)}
	\end{align*}
	Comparing the differences we get
	\begin{align*}
		&&I_{\overline{p}^{k,n}_{cp}}(k\!-\!1,n\!-\!k\!+\!2) - I_{\overline{p}^{k,n}_{cp}}(k,n\!-\!k\!-\!1)
		&\leq I_{\overline{p}^{k,n}_{cp}}(k,n\!-\!k\!+\!1) - I_{\overline{p}^{k,n}_{cp}}(k\!+\!1,n\!-\!k)
		\\
		\iff && \frac{\left({\overline{p}^{k,n}_{cp}}\right)^{k-1}(1-{\overline{p}^{k,n}_{cp}})^{n-k+1}}{(n-k+1)B(k,n-k+1)} &\leq \frac{\left({\overline{p}^{k,n}_{cp}}\right)^k(1-{\overline{p}^{k,n}_{cp}})^{n-k}}{kB(k,n-k+1)} \\
		\iff&& \frac{1-\overline{p}^{k,n}_{cp}}{n-k+1}&\leq \frac{\overline{p}^{k,n}_{cp}}{k} \\
		\iff&& \overline{p}^{k,n}_{cp} &\geq \frac{k}{n+1}
	\end{align*}
	which clearly holds since the Clopper-Pearson interval, while not necessarily symmetric around it, contains the empirical estimate $\hat{p}\geq \frac{k}{n+1}$.

	This shows $I_{\overline{p}^{k,n}_{cp}}(k,n\!-\!k\!-\!1)$ behaves concavely in a discrete sense w.r.t. $k$ since the change when increasing $k$ by $1$ is smaller than when decreasing $k$ by $1$\footnote{We suspect that the analogous continuous statement also holds, however analysing $I_{\overline{p}^{k,n}_{cp}}(k,n\!-\!k\!-\!1)$ with respect to $k$ is highly complex.}.

	Now let $\overline{p}^{k+1,n}_{cp}$ and $\overline{p}^{k-1,n}_{cp}$ be such that
	\begin{align*}
		I_{\overline{p}^{k+1,n}_{cp}}(k+1,n-k) &= 1-\delta/2, \text{ and} \\
		I_{\overline{p}^{k-1,n}_{cp}}(k-1,n-k+2) &= 1-\delta/2,
	\end{align*}
	respectively.
	Since $I_{\overline{p}^{k+1,n}_{cp}}(k+1,n-k)$ and $I_{\overline{p}^{k+1,n}_{cp}}(k+1,n-k)$ show the same discrete  concavity for the same reason as above, and $I_p(a,b)$ is non-negative and mononously increasing, we get that $\overline{p}^{k,n}_{cp}$ is discretely convex since it is defined as the inverse of $I_p(a,b)$.
	Formally, we have
	\[
	\overline{p}^{k+1,n}_{cp} - \overline{p}^{k,n}_{cp} \leq \overline{p}^{k,n}_{cp} - \overline{p}^{k-1,n}_{cp}
	\]
	which intuitively means one more positive sample increases the upper bound more than one less negative sample.
	We denote the difference between the upper bound for $k$ and $k+1$ successors as $d(k,n)$, i.e.
	\[
		d(k,n)=\overline{p}^{k+1,n}_{cp} - \overline{p}^{k,n}_{cp}
	\]
	which, as the fact above shows, is non-decreasing.

	Finally, recalling the width of the Clopper-Pearson interval width
	\[
	\varepsilon_{cp}^{k,n}= I^{-1}_{1-\delta/2}(k+1,n-k)-I^{-1}_{1-\delta/2}(n-k+1,k)-1
	\]
	we analyse whether the width increases when assuming one more positive sample:
	\begin{align*}
	&&\varepsilon_{cp}^{k+1,n}-\varepsilon_{cp}^{k,n}&\geq 0  \\
	\iff && \left(I^{-1}_{1-\delta/2}(k+2,n-k-1)-I^{-1}_{1-\delta/2}(n-k,k+1)\right) & \\
	&& - \left(I^{-1}_{1-\delta/2}(k+1,n-k)-I^{-1}_{1-\delta/2}(n-k+1,k)\right) & \geq 0 \\
	\iff && \left(I^{-1}_{1-\delta/2}(k+2,n-k-1)-I^{-1}_{1-\delta/2}(k+1,n-k)\right) & \\
	&& - \left(I^{-1}_{1-\delta/2}(n-k+1,k)-I^{-1}_{1-\delta/2}(n-k,k+1)\right) & \geq 0 \\
	\iff && d(k+1) - d(n-k) &\geq 0 \\
	\iff && k+1-(n-k) &\geq 0 \\
	\iff && k &\leq \frac{n-1}{2}
	\end{align*}
	where the second to last step follows from the fact $d(x)$ is non-decreasing.
	This shows increasing $k$ by $1$ increases the width as long $k<\frac{n-1}{2}$ (or preserves the length if $k=\frac{n-1}{2}$) showing the width is maximised at $k=\frac{n}{2}$ for even $n$ and at $\frac{n-1}{2}$ (or equivalently $\frac{n+1}{2}$) for odd $n$.
	\qed

\end{proof}

\paragraph{Relevance of \Cref{lemma:sample-complexity}.}
	In~\cite[Hypoth. 1]{DBLP:journals/tse/BuS24} it was already conjectured that the sample size for the Clopper-Pearson interval is maximal for $\hat{p}=\frac{1}{2}$.
	Since the authors were unable to prove it, they developed a \enquote{heavyweight} version of their algorithm that does not rely on their Hypothesis 1, which then has a worse performance.
	This highlights the relevance and non-triviality of \Cref{lemma:sample-complexity}.

\section{Results on the Wilson Score with CC}

\noindent
In \Cref{sec:3-sample-number} we prove for Hoeffding's inequality and the Clopper-Pearson interval that the size of the obtained confidence interval is maximal when half of the samples are positive.
Here we show that, under a light restriction, the same result also holds for the Wilson score with CC.
\begin{assumption}\label{assumption:wilson}
	For the the number of samples $n$ and $z$ as computed in the Wilson score, we have
	$n-\frac{1}{n} \geq 2 + \frac{1}{z^2}$.
\end{assumption}

Note that this assumption is indeed very light: For $\delta < 0.13$ we have $z<\frac{3}{2}$ which implies $n\geq 3$ samples are enough to fulfil the assumption.
Conversely, even for $\delta$ as large as $0.9$ (that is, allowing up to 90\% of confidence intervals to be \emph{wrong}) already $n\geq 11$ samples suffice.
Recall that in our setting the $\delta$ used in the Wilson score interval specifies the confidence budget for a \emph{single transition}.
As the transition has to share the entire confidence budget (which in many applications is often small to begin with) with all other transitions, in practice we will always have the case that $n\geq 3$ suffices.

\begin{proposition}\label{lemma:sample-complexity-wilson}
	Given a confidence budget $\delta \in (0,1]$, the confidence interval $[\underline{p},\overline{p}]$ computed by Wilson score with CC from a sequence $\stateSeq$ of $n$ samples from a binomial distribution such that \Cref{assumption:wilson} is satified maximises $\overline{p}-\underline{p}$ when $\stateSeq$ contains equally many positive and negative samples.
\end{proposition}

\begin{proof}
We first consider the case $0<\hat{p}<1$.
The maximal size of the Wilson score interval is given by $\varepsilon_{wcc}=\overline{p}_{wcc}-\underline{p}_{wcc}$ (see its definition in \cref{subsec:wilson}).
We will show that this is maximised for $\hat{p}=\frac{1}{2}$ by a standard analytic argument by showing the following two facts:
\begin{equation*}
	\frac{\partial \varepsilon_{wcc}}{\partial\hat{p}} =0 \text{ for } \hat{p}=\frac{1}{2} \quad \text{and} \quad \frac{\partial^2 \varepsilon_{wcc}}{\partial^2 \hat{p}} < 0
\end{equation*}
We start by simplifying
\begin{align*}
	\varepsilon_{wcc}&=\overline{p}_{wcc}-\underline{p}_{wcc} \\
	& = \frac{z\sqrt{z^2-\frac{1}{n}+4n\hat{p}(1-\hat{p})+(4\hat{p}-2))}+z\sqrt{z^2-\frac{1}{n}+4n\hat{p}(1-\hat{p})-(4\hat{p}-2)}+2}{2(n+z^2)}
\end{align*}
As $n$ and $z$ are positive and we are only interested in the sign of the derivatives of $\varepsilon_{wcc}$, we can instead show the above statements for $\varepsilon'_{wcc}$ defined as
\begin{equation*}
	\varepsilon'_{wcc}= \sqrt{z^2-\frac{1}{n}+4n\hat{p}(1-\hat{p})+(4\hat{p}-2))}+\sqrt{z^2-\frac{1}{n}+4n\hat{p}(1-\hat{p})-(4\hat{p}-2)}
\end{equation*}
The first partial derivative can computed through standard methods:
\begin{align*}
	\frac{\partial \varepsilon'_{wcc}}{\partial\hat{p}}= &\frac{4 n (1 - \hat{p}) - 4 n \hat{p} - 4}{ \sqrt{z^2-\frac{1}{n}+ 4 n (1 - \hat{p}) \hat{p} - 4 \hat{p} + 2}} \\
	& + \frac{ 4 n (1 - \hat{p}) - 4 n \hat{p} + 4}{ \sqrt{z^2-\frac{1}{n} + 4 n (1 - \hat{p}) \hat{p} + 4 \hat{p} - 2}}
\end{align*}
It can easily be checked that this evaluates to $0$ when $\hat{p}=\frac{1}{2}$.

The second partial derivative is also obtained by standard methods:
\begin{align*}
	\frac{\partial^2 \varepsilon'_{wcc}}{\partial^2 \hat{p}} =& -\frac{(4n(1-\hat{p})-4np+4)^2}{4(z^2-\frac{1}{n}+4n\hat{p}(1-\hat{p})+(4\hat{p}-2))^{3/2}} \\
	& -\frac{4n}{\sqrt{z^2-\frac{1}{n}+4n\hat{p}(1-\hat{p})+(4\hat{p}-2)}} \\
	& -\frac{(4n(1-\hat{p})-4np-4)^2}{4(z^2-\frac{1}{n}+4n\hat{p}(1-\hat{p})-(4\hat{p}-2))^{3/2}} \\
	& -\frac{4n}{\sqrt{z^2-\frac{1}{n}+4n\hat{p}(1-\hat{p})+(4\hat{p}-2)}}
\end{align*}
We now show that each fraction is non-negative.
For the numerators this is clear since they are either squares or $4n$.

For the denominators, we show that both $w_1=z^2-\frac{1}{n}+4n\hat{p}(1-\hat{p})+(4\hat{p}-2)$ and $w_2=z^2-\frac{1}{n}+4n\hat{p}(1-\hat{p})-(4\hat{p}-2)$ are positive.
We again use its derivative to find their extreme points:
\begin{align*}
	\frac{\partial w_1}{\partial \hat{p}} z^2-\frac{1}{n}+4n\hat{p}(1-\hat{p})+(4\hat{p}-2) &= -8n\hat{p}+4n+4 \\
	\frac{\partial w_2}{\partial \hat{p}} z^2-\frac{1}{n}+4n\hat{p}(1-\hat{p})-(4\hat{p}-2) &= -8n\hat{p}+4n-4
\end{align*}
This yields the extrema $\frac{n+1}{2n}$ for $w_1$ and $\frac{n-1}{2n}$ for $w_2$.
As these lie in the interval $(0,1)$ and are maxima, both are minimised for either the minimum or maximum $\hat{p}$.
Thus, to show they always positive, it is sufficient to show that they are positive for the smallest and largest $\hat{p}$.
Since we assumed $0<\hat{p}<1$ and by definition $\hat{p}=\frac{k}{n}$ where $k$ is whole, we have $\frac{1}{n}\leq\hat{p}\leq\frac{n-1}{n}$.
For $\hat{p}=\frac{1}{n}$ we then have $w_1=z^2-\frac{1}{n}+2>0$ and for $\hat{p}=\frac{n-1}{n}$ we have $w_1=z^2-\frac{9}{n}+6>0$.
The latter inequality follows from the fact that $0<\hat{p}<1$ implies $n\geq 2$.
For $w_2$ analogous results hold.

This shows that $\hat{p}=\frac{1}{2}$ maximises $\varepsilon_{wcc}$ over the open interval $(0,1)$.
What remains to show is that it produces a bigger confidence interval than the cases $\hat{p}=0$ and $\hat{p}=1$. Since these are symmetric, we only consider the case $\hat{p}=0$.

For the case $\hat{p}=0$ we have $\underline{p}_{wcc}=0$. Hence
\begin{align*}
	\varepsilon^{\hat{p}=0}_{wcc}&=\frac {2n{\hat {p}}+z^{2}+z{\sqrt {z^{2}-{\frac {1}{n}}+4n{\hat {p}}(1-{\hat {p}})-(4{\hat {p}}-2)}}+1}{2(n+z^{2})} \\
	&=\frac {z^{2}+z{\sqrt {z^{2}-{\frac {1}{n}}+2}}+1}{2(n+z^{2})}
\end{align*}
For the case $\hat{p}=\frac{1}{2}$ we can similarly compute
\begin{align*}
	\varepsilon^{\hat{p}=1/2}_{wcc}&\geq\overline{p}_{wcc} \\
	&=\frac {2n{\hat {p}}+z^{2}+z{\sqrt {z^{2}-{\frac {1}{n}}+4n{\hat {p}}(1-{\hat {p}})-(4{\hat {p}}-2)}}+1}{2(n+z^{2})} \\
	&\phantom{=}- \frac {2n{\hat {p}}+z^{2}-z{\sqrt {z^{2}-{\frac {1}{n}}+4n{\hat {p}}(1-{\hat {p}})+(4{\hat {p}}-2)}}-1}{2(n+z^{2})} \\
	&=\frac {2z{\sqrt {z^{2}-{\frac {1}{n}}+n}}}{2(n+z^{2})}
\end{align*}
From here, we continue by comparing the two:
\begin{align*}
	&&\varepsilon^{\hat{p}=1/2}_{wcc}&\geq\varepsilon^{\hat{p}=1/2}_{wcc}\\
	\iff&&\frac {2z{\sqrt {z^{2}-{\frac {1}{n}}+n}}}{2(n+z^{2})} & \geq \frac {z^{2}+z{\sqrt {z^{2}-{\frac {1}{n}}+2}}+1}{2(n+z^{2})} \\
	\iff&&2{\sqrt {z^{2}-{\frac {1}{n}}+n}}& \geq z+{\sqrt {z^{2}-{\frac {1}{n}}+2}}+\frac{1}{z}
\end{align*}
Notice that we only obtain any meaningful bound other than $[0,1]$ if $n\geq 2$, therefore we can assume that ${\sqrt {z^{2}-{\frac {1}{n}}+n}} \geq {\sqrt {z^{2}-{\frac {1}{n}}+2}}$.
Thus, for the above inequalities to hold it is sufficient to show that
\begin{align*}
	&&{\sqrt {z^{2}-{\frac {1}{n}}+n}}& \geq z+\frac{1}{z} \\
	\iff&&{ {z^{2}-{\frac {1}{n}}+n}}& \geq z^2+2+\frac{1}{z^2} \\
	\iff&&{ {n-{\frac {1}{n}}}}& \geq 2+\frac{1}{z^2}
\end{align*}
which holds by Assumption 1.
\end{proof}

Additionally, while \Cref{sec:3-sample-number} empirically established that Hoeffding's inequality requires more samples than the Clopper-Pearson interval for a desired precision $\varepsilon$, for the Wilson with CC we can actually show that this holds for all confidence levels $\delta$ where $\epsilon\rightarrow0$.
Note that similarly to \Cref{fig:sample-complexity}, we also suspect that the ratio between required samples for Wilson score with CC and Hoeffding's inequality only increases for larger $\varepsilon$, leading us to conjecture that the Wilson score with CC is actually more sample efficient than Hoeffding's inequality for all $\varepsilon$ and $\delta$.
\noindent
\begin{restatable}{proposition}{samplinglimit}
	\label{lemma:sampling-limit}
	Denote $n_h(\delta,\varepsilon)$ and $n_w(\delta,\varepsilon)$ the solution to the Probability Sample Complexity Problem for confidence budget $\delta$ and precision $\varepsilon$ using Hoeffding's inequality and Wilson score with CC, respectively.
	Then, under \Cref{assumption:wilson}, $r_\delta=\lim_{\varepsilon\rightarrow0}\frac{n_h(\delta,\varepsilon)}{n_w(\delta,\varepsilon)} >1$ for all $0<\delta<1$.
	Further, for all confidence budgets $\theta < \delta$, we have $r_\theta < r_\delta$.
\end{restatable}

\begin{proof}
	As established in \Cref{lemma:sample-complexity-wilson}, the sample complexity problem is solved by computing the number of samples for the case where $\hat{p}=\frac{1}{2}$ under \Cref{assumption:wilson}.
	Recall that the size of the confidence interval $\varepsilon$ is related to the worst-case sample requirements $n_{ho}$ computed by the Hoeffding bound and $n_{wcc}$ computed by the Wilson score interval with continuity correction as
	\begin{align*}
		\varepsilon =& 2 \sqrt{\frac{\log\frac{\delta}{2}}{-2n_{ho}}} \\
		\varepsilon =& \frac{z\sqrt{z^2-\frac{1}{n_{wcc}}+n_{wcc}}+1}{n_{wcc}+z^2}
	\end{align*}
	In particular, due to the symmetry of the intervals when $\hat{p}=\frac{1}{2}$, it is guaranteed that it is not necessary to clamp the functions to the interval $[0,1]$ explicitly for $\varepsilon<1$.
	Solving both for for n yields
	\begin{align*}
		n_{ho} =& \frac{\log 4 - 2\log \delta}{\varepsilon^2} \\
		n_{wcc} =& \frac{z\sqrt{(\varepsilon^2-1)(z^2\varepsilon^2-z^2-4\varepsilon)}+z^2+2\varepsilon-z^2\varepsilon^2}{2\varepsilon^2}
	\end{align*}
	From here we can directly compute the ratio $r_\delta$ by plugging in the equations and taking the limit as $\varepsilon\rightarrow 0$
	\begin{align*}
		r_\delta =&\lim_{\varepsilon\rightarrow 0} \frac{n_{ho}}{n_{wcc}} \\
		=&\lim_{\varepsilon\rightarrow 0} \frac{2(\log 4 - 2\log \delta)}{z\sqrt{(\varepsilon^2-1)(z^2\varepsilon^2-z^2-4\varepsilon)}+z^2+2\varepsilon-z^2\varepsilon^2} \\
		=& \frac{\log 4 - 2 \log \delta}{z^2}
	\end{align*}
	As $z$ is the $1-\frac{\delta}{2}$ quantile of the standard normal distribution which can be expressed in terms of the inverse error function $\text{erf}^{-1}$ as $z=\sqrt{2} \text{erf}^{-1}(2(1-\frac{\delta}{2})-1)$ (see, e.g., \cite{Jiang-MLF-2021}) we have
	\begin{equation*}
		r_\delta = \frac{\log 2 - \log \delta}{(\text{erf}^{-1}(1-\delta))^2}
	\end{equation*}
	We now claim $r_\delta > 1$ and that $r_\delta$ and monotonously increasing with $\delta$.
	To show both of these claims, we resort to an analytic argument, computing the derivative of $r_\delta$. Not that the derivative of the inverse error function is $\frac{d}{dx}\text{erf}^{-1}(x)=-\frac{e^{(\text{erf}^{-1})^2}}{\sqrt{\pi}}$ (see, e.g., \cite{bergeron_labelle_leroux_1997}).
	Using this fact we obtain the derivative of $r_\delta$ by standard calculus:
	\begin{equation*}
		\frac{d r_\delta}{d \delta} = -\frac{\text{erf}^{-1}(1-\delta)+\sqrt{\pi} \delta e^{(\text{erf}^{-1}(1-\delta))^2} \log \frac{\delta}{2}}{\delta \text{erf}^{-1}(1-\delta)^3}
	\end{equation*}
	We now claim this is positive for $0<\delta< 1$.
	Since $0<\delta< 1$ implies ${\delta \text{erf}^{-1}(1-\delta)^3}>0$, it is sufficient to prove
	\begin{equation*}
		-\text{erf}^{-1}(1-\delta)-\sqrt{\pi} \delta e^{(\text{erf}^{-1}(1-\delta))^2} \log \frac{\delta}{2} > 0
	\end{equation*}
	Note that all terms except $\log\frac{\delta}{2}$ are positive.
	We reformulate the statement such that only positive terms occur.
	\begin{equation*}
		\sqrt{\pi} \delta e^{(\text{erf}^{-1}(1-\delta))^2} \log \frac{2}{\delta} > \text{erf}^{-1}(1-\delta)
	\end{equation*}
	As $\frac{1}{\sqrt{2e}} e^{x^2}>x$ for all $x$, we can already conclude the statement holds if $\sqrt{\pi} \delta log \frac{2}{\delta} \geq \frac{1}{\sqrt{2e}}$ which can be solved to be the case  when $0.0731...\leq\delta\leq 1$ on the interval $[0,1]$.

	To show that $\frac{d r_\delta}{d \delta}$ is always positive, we now focus on the interval $(0,0.0732]$.
	We aim to show that the second derivative $\frac{d^2 r_\delta}{d^2 \delta}$ is negative on the interval $(0,0.0732]$.
	This is sufficient to prove the claim as we have already shown that $\frac{d r_\delta}{d \delta}>0$ for $\delta=0.0732$.

	\begin{align*}
		\frac{d^2 r_\delta}{d^2 \delta} = & -\frac{1}{2\delta^2 \text{erf}^{-1}(1-\delta)^4}\Bigg( \pi\delta^2e^{2\text{erf}^{-1}(1-\delta)^2}(3-2\text{erf}^{-1}(1-\delta)^2)\log\frac{\delta}{2}\\
		& +4\sqrt{\pi}\delta e^{\text{erf}^{-1}(1-\delta)^2}\text{erf}^{-1}(1-\delta)+2\text{erf}^{-1}(1-\delta)^2 \Bigg)
	\end{align*}
	Analysing the signs of the terms, we see that $-\frac{1}{\delta^2\text{erf}^{-1}(1-\delta)^2}$ is negative, which means we have to show the term in the bracket is positive.
	We can indeed show that all summands are positive: For $4\sqrt{\pi}\delta e^{\text{erf}^{-1}(1-\delta)^2}\text{erf}^{-1}(1-\delta)$ and $2\text{erf}^{-1}(1-\delta)^2$ this is easy since all factors are positive.
	For the factors of the remaining summand, $\pi\delta^2e^{2\text{erf}^{-1}(1-\delta)^2}(3-2\text{erf}^{-1}(1-\delta)^2)\log\frac{\delta}{2}$, we can see that $\pi\delta^2e^{2\text{erf}^{-1}(1-\delta)^2}$ is positive, whereas $\log\frac{\delta}{2}$ is negative.
	It remains to show that $(3-2\text{erf}^{-1}(1-\delta)^2$ is negative.
	As we only have to prove this for $(0,0.0732]$ and the term is positively correlated with $\delta$, we obtain that $(3-2\text{erf}^{-1}(1-\delta)^2$ is maximal on the interval $(0,0.0732]$ if $\delta=0.0732$ which yields $(3-2\text{erf}^{-1}(1-0.0732)^2\approx -0.21 < 0$.

	This finishes the proof of the claim that $\frac{d r_\delta}{d \delta} > 0$.
	From this it immediately follows that $r_\delta$ is positively correlated with $\delta$.
	It remains to show that $r_\delta > 1$.
	As we know $\frac{d r_\delta}{d \delta} > 0$, showing $r_0\geq 1$ is sufficient to show that $r_\delta> 1$ for all $\delta > 0$.
	Unfortunately, we cannot directly finish the proof by computing $r_0$ since it is not defined.
	However, we can find its limit:
	\begin{equation*}
		\lim_{\delta\rightarrow 0} r_\delta = \lim_{\delta\rightarrow 0} \frac{\log 2-\log \delta}{\text{erf}^{-1}(1-\delta)} = \lim_{\delta\rightarrow 0} \frac{\log 2-\log \delta}{\text{erfc}^{-1}(\delta)} = 1
	\end{equation*}
	Here, $\text{erfc}^{-1}$ denotes the inverse complementary error function, defined as $\text{erfc}^{-1}(x)=\text{erf}^{-1}(1-x)$ \cite{Strecock68}.
	The limit follows from the asymptotic expansion of $\text{erfc}^{-1}(x)$ for $x\rightarrow 0$ which is of order $-\log {x}$ \cite{Strecock68}.
	This shows $r_\delta\geq 1$ for all $0 < \delta < 1$, finishing the proof.\qed
\end{proof}

\section{Dependence of Sample Complexity Ratio on $\delta$}
\label{sec:app-comparison}
\label{sec:app-sample-delta}

\noindent

\begin{figure}[t]
	\centering
		\includegraphics[width=\textwidth,keepaspectratio]{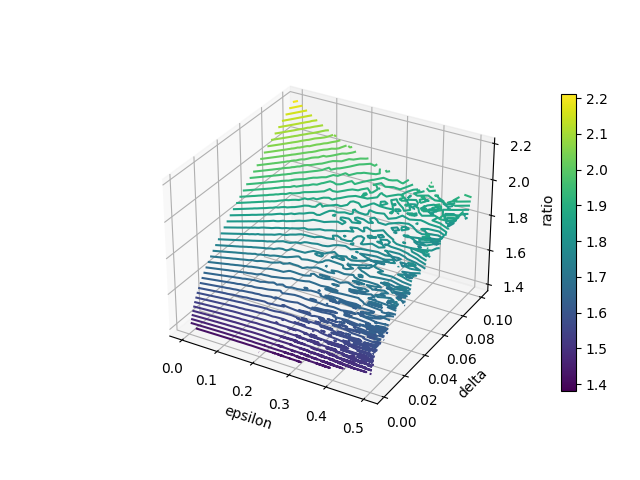}
		\caption{Ratio of required samples for Hoeffding bound and Clopper-Pearson interval for varying precision requirement $\varepsilon$ and confidence budgets $\delta$ in the worst-case ($\hat{p} = \frac{1}{2}$).}
	\label{fig:sample-complexity-3d}
\end{figure}

\noindent
In \Cref{sec:3-sample-number} we discuss the ratio between sample complexities for the Hoeffding bound and Clopper-Pearson interval, respectively.
For an empirical approach, \Cref{fig:sample-complexity} plots these ratios for a range of different precisions $\varepsilon$ and a selected confidence budgets $\delta=0.1$.
A question that remains is how the ratio behaves for other values of $\delta$.
In \Cref{fig:sample-complexity-3d} we plot the ratios between the sample complexity of the two methods while varying both $\varepsilon$ and $\delta$.

First of all, \Cref{fig:sample-complexity-3d} shows once again that the Hoeffding bound always requires more samples than Clopper-Pearson interval in the worst-case, as the ratio between the sample complexity of the Hoeffding bound and the Clopper-Pearson interval is larger than 1 for all relevant values of $\varepsilon$ and $\delta$.
Moreover, the ratio is not monotonically related to $\varepsilon$, but it shifts in favour of the Clopper-Pearson interval for small $\varepsilon$.
Finally, the positive correlation with $\delta$ means it is even more valuable to increase the confidence budget using the structural improvement methods from \Cref{sec:4-title} since these allow for higher $\delta$ per transition.

\section{Technical Details of the Structural Improvements (\cref{sec:4-title})}\label{app:structural}

\subsection{\independence}\label{app:independence}
\noindent
By their very nature, the transition distributions in Markov systems are independent.
Hence, instead of employing the union bound when estimating the \emph{overall} error probability, when can multiplicatively divide up the confidence, as follows.
Recall that ultimately we are interested in the probability of collectively estimating all probabilities of every distribution we are interested in correctly.
More formally, let $\eventAllCorrect= \text{\enquote{All transitions estimated correctly}}$ and require that $\Probability[\eventAllCorrect] \geq 1 - \delta$.
Observe that this event equals $\Intersection_{t \in \mdptransitions} \eventTransCorrect{t}$, where $\eventTransCorrect{t}$ is the event that transition $t$ is estimated correctly.
As explained in \Cref{sec:2-smc}, the usual trick is to consider the complement of $\eventAllCorrect$ and use the union bound, formally
$\delta \leq 1 - \Probability[\eventAllCorrect] = \Probability[\setcomplement{\eventAllCorrect}] = \Probability[\Union_{t \in \mdptransitions} \setcomplement{\eventTransCorrect{t}}] \leq \sum_{t \in \mdptransitions} \Probability[\setcomplement{\eventTransCorrect{t}}]$.

Instead of splitting $\eventAllCorrect$ individually by transitions, we propose to first split by state-action pairs, i.e.\ $\Probability[\eventAllCorrect] = \Probability[\Intersection_{(s, a) \in \States \times \Actions} \eventStateActionCorrect{s}{a}]$, where $\eventStateActionCorrect{s}{a}$ is the event that we correctly estimate the overall distribution $\mdptransitions(s, a)$.
Intuitively, we simply say that the probability of estimating everything correctly equals the probability of jointly for each distribution estimating every transition of that distribution correctly.
Now, note that while the events $\eventTransCorrect{t}$ are not independent, $\eventStateActionCorrect{s}{a}$ however are.
Thus, we obtain that $\Probability[\eventAllCorrect] = \prod_{(s, a) \in \States \times \Actions} \Probability[\eventStateActionCorrect{s}{a}]$.

To generalise, note that it is irrelevant that we talk about state-action pairs and their related distributions.
Rather, suppose that we are interested in estimating each distribution in a given set of independent distributions $\mathcal{D}$, e.g.\ $\{\mdptransitions(s, a) \mid (s, a) \in \States \times \Actions\}$, and define $\eventAllCorrect_{\mathcal{D}}$ and $\eventAllCorrect_{d}$ as the event that we estimated all distributions and distribution $d \in \mathcal{D}$ correctly, respectively.
Then,
\begin{lemma}
	We have $\Probability[\eventAllCorrect_\mathcal{D}] = \prod_{d \in \mathcal{D}} \Probability[\eventAllCorrect_d]$.
\end{lemma}
\begin{proof}
	By the above argument.\qed
\end{proof}
From this, we also obtain our original statement.
\lemmaindependence*
While this restriction is strictly better, we note that it has diminishing effects for large number of distributions to be estimated or strong confidence requirements, i.e.\ $\delta$ close to $0$.
In particular, we have that the relative difference of confidence requirements imposed on each distribution, i.e.\ $(1 - \delta)^{1/n} / (1 - \delta / n)$, approaches $1$ for $\delta \to 0$ or $n \to \infty$.
Dually, this also means that the effect is more prominent if we can weaken the confidence requirements for particular distributions or shrink $\mathcal{D}$ altogether, using e.g.\ our other presented methods.

\subsection{\fragments}\label{app:fragments}

\noindent
Here, we discuss the \fragments optimization to identify parts state space for which the \enquote{internal} behaviour is not (too) interesting, and thus can be abstracted using a quotient construction.
We suggest to read the illustrative \Cref{ex:4-fragments} before continuing.
Observe that in the example, if we had a second transition entering this area (for example at state $s_2$), we would need to estimate four probabilities instead of two.
This is because for every entering transition (and every internal strategy), we have a different distribution over the exits.

Generally, for a set of states $R \subseteq \States$, we can estimate either the internal behaviour, i.e.\ each transition probability individually, or for each pair of entry and internal strategy estimate the distribution over all exits, which we define as follows.
A state $s \in R$ is called \emph{entry of $R$} if there exists a transition $(s', a, s) \in \mdptransitions$ with $s' \notin R$ or if $s$ is the initial state.
Similarly, $s \notin R$ is called \emph{exit of $R$} if there exists a transition $(s', a, s) \in \mdptransitions$ with $s' \in R$.
Together, let $\entries$ and $\exits$
refer to the set of all entries and exits
of $R$, respectively.
We assume for any strategy the probability to eventually leave $R$ is 1, i.e.\ it is not an end component (which we can detect and treat differently, namely collapse into a single state using \equivalencestructures).

Now, observe that fixing an action in each state induces a distribution over exits for each entry.
In other words, we conceptually replace the set of actions in each entry state by each possible assignment of actions to all states in the fragment, i.e.\ $\StrategiesMD<\MDP>(R) \coloneqq \{\pi \mid \pi : R \to \Actions \land \forall s \in R.~\pi(s) \in \stateactions(s)\}$.
Intuitively, this means upon entering the fragment, we can choose among all internal behaviour and then perform one \enquote{macro step}, skipping over all the internal steps and directly moving to one of the exits.
(Practically, this means sampling a path under the chosen internal behaviour until one of the exits is reached.)
We can then estimate the transition probabilities for each of these \enquote{macro actions}.
This is exactly captured by \cref{def:4-fragments}.

\lemmafragments*

\begin{proof}
	First, observe that the set of paths used to derive the new transition probabilities $\mdptransitions'(s, \pi, e)$ exactly correspond to the \emph{until} query $R \mathop{\textbf{U}} \{e\}$, and thus is measurable, see e.g.\ \cite[Sec. 7]{DBLP:conf/sfm/ForejtKNP11}.

	Next, since we required that $R$ does not contain an EC, we know that $\ProbabilityMDP<\MDP, s><\strategy>[\reach \setcomplement{R}] = 1$ under any strategy $\strategy$.
	In other words, let $E_n^e = \{\infinitepath \mid \forall i < n.~\infinitepath_i \in R \land \infinitepath_n = e\}$ the set of finite paths that exit $R$ in $e$ at step $n$.
	Then $\ProbabilityMDP<\MDP, s><\strategy>[\Union_{n = 1}^\infty \Union_{e \in \States \setminus R} E_n^e] = 1$, observing that $E_n^e$ are pairwise disjoint.
	(Note that we cannot require that $E_n^e$ exactly partitions the set of paths, there may be a measure-zero set of paths that forever remain in $R$.)
	Moreover, $\ProbabilityMDP<\MDP, s><\strategy>[R \mathop{\textbf{U}} \{e\}] = \ProbabilityMDP<\MDP, s><\strategy>[\Union_{n=1}^\infty E_n^e]$.

	Now, we argue that if for two strategies $\strategy, \strategy'$ we have that $\strategy(s) = \strategy'(s)$ for all $s \in R$, then $\ProbabilityMDP<\MDP, t><\strategy>[E_n^e] = \ProbabilityMDP<\MDP, t><\strategy'>[E_n^e]$ for any $n$ and $e$ (and consequently the overall probability of exiting at $e$ is the same).
	This follows immediately by observing that $E_n^e$ can be written as disjoint union of cylinder sets induced by finite paths of length exactly $n$ which remain in $R$ for $n - 1$ steps and then transition into $e$.
	Since the two strategies yield the same decisions for all states in $R$, the probability measure of these cylinder sets equals, proving the result.
	Therefore, while remaining in $R$, any element of $\StrategiesMD<\MDP>(R)$ corresponds to a strategy in $\MDP$.

	Now, to obtain the overall equality, observe that picking the supremum over all strategies in $\MDP'$ chooses some strategy $\strategy'$ as action in $s$.
	This can be immediately replicated by just choosing $\strategy'(s)$ for all $s \in R$ in $\MDP$.
	Dually, any optimal strategy $\strategy$ for $\MDP$ is available as action in $s$.
	Thus, the value is preserved.\qed
\end{proof}
\begin{remark}
	We note that while this optimization may save a lot of confidence budget, there is a subtlety in terms of the number of samples that might make it better to learn the distributions individually even if the above approach would recommend otherwise.
	Assume we have a fragment that is very large, where exits are only reached with very small probability; as an extreme example, consider the Markov chain in~\cite[Fig. 3]{HM18}; there is a single entry and two exits, but paths to the exits are typically of exponential length.
	If we have full sampling access, we can spend linear time in the size of the fragment to estimate all transition probabilities.
	If, in this setting, we \enquote{collapse} the fragment, our mode of getting samples changes, because now we fix an internal behaviour and then sample the fragment until an exit is taken.
	Thus, it can now take exponential time to obtain a single sample for the collapsed fragment, but all the exponentially many internal steps are not used for the estimation.
	We highlight that acyclic fragments (including chains) never have this problem.
\end{remark}
Now, two questions remain:
\begin{enumerate}[(1)]
	\item Given a candidate-fragment $R$, should we apply this technique?
	\item How do we find such candidates $R$?
\end{enumerate}
For the Question (1), we give a simple answer:
Concretely, we can estimate how \enquote{costly} the two alternatives are by computing how many distinct probabilities we would have to learn.
For the classical approach, this effectively boils down to the number of transitions originating from states in $R$.
In contrast, for the above idea, we need to learn (at most) $\cardinality{\entries} \cdot \prod_{s \in R} \cardinality{\stateactions(s)} \cdot \cardinality{\exits}$ probabilities:
For each entry state and each choice of internal actions, we need to learn the distribution over all exits.
(The fact that only internal actions inside $R$ are relevant follow from the proof of \Cref{stm:fragments}.)
We then choose the \enquote{cheaper} option, i.e.\ the one where we need to learn fewer probabilities, which allows us to provide tighter bounds using the same number of samples.
Observe that the expression above only is an upper bound:
Firstly, for a particular choice of actions, some exits may not be reachable at all, which can be detected by graph analysis.
Secondly, it may happen that for a concrete choice of actions, an internal state $s' \in R$ is not reached, which would also imply that for this particular set of actions the action chosen in $s'$ is not relevant and we would not need to learn the distribution over exits for each possible choice in $s'$.
Thirdly, we can also account for \smallsupport{} cases.

For Question (2), we highlight three possible choices.
\begin{description}
	\item[Chains]
	As a simple choice, we propose to choose $R = \{s\}$ for all states with a single predecessor, i.e. there is a unique state-action pair $(t, a)$ with $\mdptransitions(t, a, s) > 0$.
	Then, instead of learning $\mdptransitions(t, a, s)$ and $\mdptransitions(s, a', s')$ for all $a' \in \stateactions(s')$, we intuitively learn $\mdptransitions(t, a, s) \cdot \mdptransitions(s, a', s')$, which is one less transition.
	Our experimental evaluation shows that applying this process performs surprisingly well.

	\item[SCC]
	A strongly connected component (SCC) is defined as follows:
	A non-empty set of states $C \subseteq \States$ in an MDP is \emph{strongly connected} if for every pair $s, s' \in C$ there is a finite path (a finite prefix of infinite paths ending in a state, i.e.\ elements of $(\States\times\Actions)^\ast \times \States$) from $s$ to $s'$.
	Such a set $C$ is a \emph{strongly connected component} (SCC) if it is inclusion maximal, i.e.\ there exists no strongly connected $C'$ with $C \subsetneq C'$.

	We consider each SCC of the system as candidate, as they are easy to find and induce a topological ordering of the state space~\cite{DBLP:journals/siamcomp/Tarjan72}.

	\item[Global Search]
	We could phrase the above problem as an optimization problem and try to find (near-)optimal candidates.
	As our goal is to offer methods with essentially no computational overhead, we refrain from this approach and leave it for future work.
	However, despite the combinatorial hardness, such an optimization may particularly pay off in settings where samples are extremely costly to acquire.
\end{description}
The complexity of identifying fragments naturally depends on the chosen class of candidates.
Identifying chains and SCCs can be done in linear time, while we conjecture global search to be NP-hard.

\subsection{Other Objectives}\label{app:other-objectives}
MDPs can be complemented with various objectives, e.g., reachability, mean payoff, (discounted) total reward, or linear temporal logic.
In the typical SMC algorithm (\cref{sec:2-smc}), the objective is relevant when solving the interval MDP.
As our work focusses on the step of estimating the intervals, the objective hardly affects our methods.
In particular, everything discussed in \cref{sec:3-title} and \cref{sec:4-use-model} (using Clopper-Pearson, \independence and \smallsupport) is completely independent of the objective.
Only those improvements which specifically utilize information about the property depend on the objective, namely \equivalencestructures and \fragments.
We discuss how to modify them for other choices of objective.

To generalise the idea of \equivalencestructures and \fragments, observe that key idea is that internal behaviour of the equivalence class or fragment, respectively, does not affect the value of the objective.
Thus, for prefix-independent infinite-horizon objectives such as mean payoff or $\omega$-regular ones, we can apply \fragments without major modifications, as well as directly use the special case of essential states in \equivalencestructures.
For prefix-dependent objectives such as total reward, we need to ensure that all paths through the fragment affect the objective in the same way, e.g.\ obtain the same reward.
In particular, we may search for such structures among all states with reward zero.

To collapse MECs and their attractors as part of \equivalencestructures, we first need to be able to determine the value of states within the MEC.
By definition of a MEC, we can visit every state of the MEC infinitely often from every other state inside the MEC and its attractor.
For $\omega$-regular objectives, we can determine the long-run behaviour (the MEC is either \enquote{winning} or \enquote{losing}) from the structure alone using standard methods.
For total reward with only non-negative or non-positive rewards, we can also determine the value by considering whether the MEC contains some non-zero reward.
For MECs with positive and negative rewards, as well as for mean payoff objectives, the value of the MEC (if it exists), depends on the exact transition probabilities within the MEC and can therefore not be simplified a-priori.
Nonetheless, for mean-payoff the attractor of a MEC can be collapsed, likewise for total reward the attractor of a MEC among only zero-reward states can also be collapsed.

In theory, the NWR can be extended to other objectives as well.
However, since it is already coNP-complete for reachability \cite{LeRPer18} and even coETR-complete for weighted reachability \cite{EngPerRao23}, it is unlikely there are any relevant objectives for which the complete NWR can be computed efficiently.

\section{Additional Details on the Evaluation}\label{app:exp-details}

\noindent
Here, we show the full versions of \Cref{tbl:results,tbl:ablation} referenced in \Cref{sec:5-title}.
The values of the parameters are in order as in the PRISM benchmark suite.
Throughout the entire section we use a confidence budget of $\delta=0.1$.

\begin{table}[p]
	\centering
	\caption{\textbf{Property-SMC} results including parameters and runtimes} \label{tbl:results-full}
	\begin{adjustbox}{angle=270}
	\begin{tabular}
		{ > {\ttfamily}l > {\ttfamily}l|rrr|rrr|rr}
		& & & \multicolumn{2}{c}
		{Transitions} &  \multicolumn{3}{c}
		{Samples} & \multicolumn{2}{c}
		{Runtime}        \\\cline{6 - 8}
		\normalfont
		Model / Property & \normalfont Parameters & $\varepsilon$ & Base & Ours & Base & Ours & Factor & Base & Ours \\
		\midrule
		consensus.disagree & 2-2 & 0.3 & 484 & 352 & 38\,090.8 & 14\,766.4 & 2.6 & 8.0 & 3.5 \\
		csma.all{\textunderscore}before & 2-2 & 0.1 & 1276 & 14 & 64\,310.4 & 783.0 & 82.1 & 20.0 & 0.9 \\
		firewire{\textunderscore}dl.deadline & 3-200 & 0.05 & 17\,417 & 63 & 1\,685\,965.6 & 46\,286.4 & 36.4 & 936.9 & 27.8 \\
		pacman.crash & 5 & 0.1 & 84 & 2 & 5048.4 & 180.6 & 28.0 & 0.6 & 0.5 \\
		wlan.collisions{\textunderscore}max & 2 & 0.7 & 5444 & 716 & 354\,948.8 & 9118.4 & 38.9 & 95.3 & 3.5 \\
		wlan{\textunderscore}dl.deadline & 0 & 0.9 & 326\,883 & 14\,847 & 21\,083\,953.5 & 201\,535.6 & 104.6 & 8123.2 & 101.5 \\
		zeroconf.correct & 20-2-true & 0.05 & 953 & 265 & 18\,202.3 & 163.9 & 111.1 & 1.8 & 0.5 \\
		zeroconf{\textunderscore}dl.deadline & 1000-1-true-10 & 0.05 & 4777 & 2 & 48\,247.7 & 64.4 & 749.2 & 3.2 & 0.7 \\
	\end{tabular}
	\end{adjustbox}
\end{table}

\begin{table}[p]
	\caption{Ablation study of \textbf{Property-SMC} per problem instance\label{tbl:ablation-full}}
	\centering
	\begin{adjustbox}{angle=270}
		\setlength\tabcolsep{4pt}
\begin{tabular}{>{\ttfamily}l>{\hskip 2pt\ttfamily}lr>{\hskip 5pt}rrrrrrr}
	\normalfont Model / Property                             & \normalfont  Parameters & Transitions & Baseline  & CP    & \smallsupport & \independence  & Structure & NWR   & Chains \\ \midrule
	consensus.disagree & 2-2 & 484 & 2.580 & 1.092  & 1.793   & 1.030 & 1.064  & 1.002 & 1.062  \\
	\midrule
	csma.all{\textunderscore}before & 2-2 & 1276 & 82.133 & 2.008  & 1.166   & 1.079 & 20.297  & 10.397 & 4.667  \\
	\midrule
	firewire{\textunderscore}dl.deadline & 3-200 & 17\,417 & 36.425 & 1.630  & 1.293   & 1.247 & 3.919  & 2.164 & 1  \\
	\midrule
	pacman.crash & 5 & 84 & 27.953 & 2.558  & 1.578   & 1.578 & 11.661  & 11.445 & 7.888  \\
	\midrule
	wlan.collisions{\textunderscore}max & 2 & 5444 & 38.927 & 1.948  & 1.633   & 1.024 & 9.301  & 4.590 & 1.204  \\
	\midrule
	wlan{\textunderscore}dl.deadline & 0 & 326\,883 & 104.617 & 1.934  & 1.599   & 1.012 & 25.837  & 16.307 & 1.281  \\
	\midrule
	zeroconf.correct & 20-2-true & 953 & 111.057 & 10.000  & 1.617   & 1.818 & 21.867  & 7.373 & 1.221  \\
	\midrule
	zeroconf{\textunderscore}dl.deadline & 1000-1-true-10 & 4777 & 749.188 & 14.708  & 1.289   & 1.289 & 252.419  & 143.385 & 1  \\
\end{tabular}
	\end{adjustbox}
\end{table}

\paragraph{Setup.} \label{app:prop-smc-setup}
To determine a suitable precision $\varepsilon$ for each problem instance (i.e.\ such that the baseline SMC algorithm terminates in a reasonable amount of time), we collected sample runs equal to 100 times the number of probabilistic transitions in the model, where each sample run is a path starting in the initial state and reaches either a goal state or a sink state under the scheduler that uniformly at random picks an action in each state.
We then computed the precision obtained by our baseline method (i.e.\ uniform distribution of confidence budget $\delta$ over all transitions and using Hoeffding's inequality).
Finally, we rounded the obtained precision on the value to the nearest $0.05$ and used them as the desired precision (see column $\varepsilon$ in \Cref{tbl:results,tbl:results-full}).

To evaluate the impact of our improvement, we sampled random paths from the MDP, again using a uniform random scheduler, and estimate the value of the property by an interval $[\underline{p},\overline{p}]$, once with and once without our improvements.
For comparable results, we use the same set of samples with the same outcomes for both algorithms.
We continue sampling random paths for both estimates until we have $\overline{p}-\underline{p}\leq \varepsilon$ for the respective estimate.
Under \enquote{Samples} we then report how many sample paths were required to obtain the given precision in each of the two settings.
Further, to express the model size, we give the number of transitions in the model under ``Transitions'' in column ``Base'' in \Cref{tbl:results-full}.
In column ``Ours'' we give the number of transitions in the model after applying the structural improvements of \Cref{sec:4-use-prop}, i.e. \equivalencestructures and \chainfragments.
We emphasize that less transitions does not skew the sample count here since we specifically count sample \emph{runs}.
This is in line with the motivation of SMC since the transformed model is only an equivalent representation which we do not have direct sampling access to.
The smaller size only indirectly improves model-based SMC by allowing for a higher confidence budget per transition, as well as speeding up the solving of the inferred interval MDP.

Finally, in the columns \enquote{runtime} we give the runtime of both the baseline algorithm, and our improved algorithm in seconds.
Note that we only report the time required to compute $\varepsilon$ with the final data set sufficient to achieve a precision of $\varepsilon$, not the time required to find the required data set size.
This can be seen as the algorithm having access to an oracle deciding the number of sample runs required to achieve the desired precision.
The reason for this is that the overall runtime may be improved by optimizing the way in which the sampling and estimation phases of the SMC algorithm (see \Cref{sec:2-smc}) are interleaved, so we report the best-case runtime if enough samples are immediately gathered in the first step.

\paragraph{Ablation study.}
For each improvement, we repeat the process explained above, but without including that improvement.
We again give the improvement factor, i.e.\ ratio of required samples, between using all improvements and using all improvements except one in \Cref{tbl:ablation-full}.
For the structural improvements of \Cref{sec:4-use-prop} we do three separate test: One where we use no structural improvements at all (column ``Structure''), and one where we use \chainfragments, but not \equivalencestructures, and vice verse (columns ``NWR'' and ``Chains'', respectively).
Keep in mind that the heading says which improvement we did \emph{not} use.

\paragraph{Evaluation.}
We first emphasize that \emph{all} factors in our table are $\geq 1$.
In other words, employing our improvements never makes the sample complexity worse.
For \smallsupport, \independence, and \equivalencestructures this is clear since, if applicable, they only allow us to use a higher confidence budget per transition.
With the results of \Cref{sec:3-sample-number}, it is also not surprising that using the Clopper-Pearson interval reduces the number of samples required, although we were not able to formally prove that this holds for all confidence budgets $\delta$ and sample sizes $n$.
Lastly, while \chainfragments always reduces the number of transitions in the model, this does in theory not guarantee that less samples are required to estimate a \emph{property}.
However, we did not encounter any cases where applying \chainfragments had a negative effect in our evaluation.

We now discuss the impact of each improvement.
As \Cref{fig:sample-complexity} (right) suggests, Clopper-Pearson gains more of an advantage over Hoeffding's inequality the more $\hat{p}$ is away from $\frac{1}{2}$.
This is if course likely the case if the real probabilities in the underlying model are away from $\frac{1}{2}$.
Notably, both \texttt{zeroconf} models have transition probabilities of order $10^{-3}$ which is likely why they exhibit the highest improvement factor here.
Further, by \Cref{fig:sample-complexity-3d} we expect a bigger advantage for Clopper-Pearson for models where more confidence budget is available per transition.
This is exactly the case for the models with least transitions, i.e.\ \texttt{zeroconf\_dl} and \texttt{pacman} (see column ``Transitions / Ours'' in \Cref{tbl:results-full}) where indeed we observe comparably higher improvement factors.

Following the discussion after \Cref{prop:independence} in \Cref{app:independence}, we expect a similar statement to hold, in that the impact of \smallsupport diminishes for small confidence budgets, i.e.\ a large number of transitions.
We empirically confirm this by noticing that again both \texttt{zeroconf\_dl} and \texttt{pacman} have comparably high improvement factors.
Additionally, we notice that \texttt{zeroconf} actually has the highest improvement factor despite still having 265 transitions in the transformed model.
We conjecture two reasons for this:
First, some of the transitions may be deterministic and thus not require any confidence budget by \smallsupport.
Second, the difference in confidence budget may have a larger impact for the aforementioned small transition probabilities.

For \smallsupport itself, the impact of course mostly depends on the number of deterministic, or binary transitions and thus depends on the structure of the model.

Similarly, predicting the impact of the structural improvements from \equivalencestructures and \chainfragments, respectively, is non-trivial and requires analysis of the model structure.
In general we notice that \equivalencestructures can have a very large impact.
This does not only depend on the model, but also the property since many transitions can be deemed irrelevant with respect to the property and thus be collapsed.
Lastly, we comment on the synergy between \equivalencestructures and \chainfragments:
Looking at, e.g.\ \texttt{pacman}, we see that, despite including \chainfragments reduces the sample count by a factor of almost 8, not utilising \chainfragments has little impact when also excluding \equivalencestructures.
This can be seen in the very small difference between columns ``Structure'' (where neihter \chainfragments not \equivalencestructures) are utilised) and ``NWR'' (where \chainfragments is utilised but not \equivalencestructures).
We explain this by noticing that both are not mutually exclusive since a chain may, e.g. be part of a MEC and thus equally handles by \equivalencestructures.
In the extreme case, \chainfragments may not seem have any impact, such as in \texttt{zeroconf\-dl} and \texttt{firewire\_dl}, as the improvement factor is $1$.
While there are still chains present, as shown by the difference between columns ``Structure'' and ``NWR'', the NWR actually subsumes all chains in these examples.
On the other hand, there is also the possibility that \equivalencestructures and \chainfragments may synergise: For example a state may have two incoming transitions from two states in the same MEC.
Then it is only a candidate for a chain once the MEC is collapsed via \equivalencestructures.
However, the ablation study alone cannot answer whether this is the case, or whether the overlap of \equivalencestructures and \chainfragments explains superlinear behaviour.
This requires an analsis of the structure of the specific models.
Lastly we want to mention that this synergy may only work in this one direction:
Since \equivalencestructures are defined via the semantics of the MDP and fragments do not alter the semantics (see \Cref{stm:fragments}), collapsing fragments cannot extend the NWR used in \equivalencestructures.

As a final caveat, we comment on the sampling strategy.
Since the problem of finding a good sampling strategy is orthogonal, we used a uniform sampling strategy here to focus on the impact of our improvements on the process of statistical inference.
In the context of a full SMC, we conjecture our methods will potentially have an even larger impact on the sample complexity since suboptimal strategies may be recognized earlier, leading to more samples for the optimal strategy which, again using our improvements, can be utilised to greater effect.

\end{document}